\documentclass[lettersize,journal, onecolumn]{IEEEtran}
\usepackage[T1]{fontenc} 
\usepackage{amsmath,amsfonts}
\usepackage{amsthm}
\usepackage{amssymb}
\usepackage{bm}
\usepackage{mathrsfs}
\usepackage{multirow}
\usepackage{algorithmic}
\usepackage{algorithm}
\usepackage{array}
\usepackage[caption=false,font=normalsize,labelfont=sf,textfont=sf]{subfig}
\usepackage{textcomp}
\usepackage{stfloats}
\usepackage{url}
\usepackage{verbatim}
\usepackage{graphicx}
\usepackage{cite}
\usepackage{float}
\usepackage{enumerate}
\usepackage{enumitem}
\usepackage[table]{xcolor}

\makeatletter

\newcommand{\Rmnum}[1]{\expandafter\@slowromancap\romannumeral #1@}
\makeatother

\newtheorem{theorem}{Theorem}

\newtheorem{proposition}{Proposition}
\newtheorem{lemma}{Lemma}
\newtheorem{remark}{Remark}

\newtheorem{assumption}{Assumption}

\begin{document}
	\title{Force-coordination Control for Aerial Collaborative Transportation based on Lumped Disturbance Separation and Estimation}
	
	\author{Lidan Xu, 
		Hao Lu,
		Jianliang Wang, 
		Xianggui Guo,
		and Lei Guo, 
		\thanks{Manuscript received XXX, 202X; revised XXXX, 202X.}
		\thanks{This work was supported by the National Natural Science Foundation of China under Grants 62173024, 62273024, 
			 Zhejiang Natural Science Foundation under Grants LD21F030001, LZ22F030012, and  the Program for Changjiang Scholars and Innovative Research Team under Grant IRT\underline{\hspace{0.5em}}16R03.}
		\thanks{Lidan Xu and Lei Guo are with School of Cyber Science and Technology, Beihang University, Beijing, 100191, China.}
		\thanks{Hao Lu (*Corresponding author, E-mail: luhaojiqi@126.com) and Jianliang Wang are with Hangzhou Innovation Institute, Beihang University, Zhejiang, 310052, China.}
		\thanks{Xianggui Guo is with the School of Automation and Electrical Engineering, University of Science and Technology Beijing, Beijing, 100083, China.}}
	
	\markboth{xxxx}%
	{Lidan Xu \MakeLowercase{\textit{et al.}}: A Sample Article Using IEEEtran.cls for IEEE Journals}
	
	
	\maketitle
	
	\begin{abstract}
		This article studies the collaborative transportation of a cable-suspended pipe by two quadrotors.
		 A force-coordination control scheme is proposed, where a force-consensus term is introduced to average the load distribution between the quadrotors.
		 Since thrust uncertainty and cable force are coupled together in the acceleration channel, 
		 disturbance observer can only obtain the lumped disturbance estimate.
		 Under the quasi-static condition, a disturbance separation strategy is developed to remove the thrust uncertainty estimate for precise cable force estimation.
		The stability of the overall system is analyzed using Lyapunov theory.
		Both numerical simulations and indoor experiments using heterogeneous quadrotors validate the effectiveness of thrust uncertainty separation and force-consensus algorithm.
	\end{abstract}
	
	
	\section{Introduction}
	\IEEEPARstart{R}{ecent} years have witnessed increasing application in the area of drone delivery.
	The bottleneck problem of aerial transportation lies in the limitation of the payload capacity.
	Although using a larger vehicle may solve this problem,
	it is believed to be more costly and inefficient.
	As described in \cite{Carter}, when the size of a rotorcraft  increases to a certain point, 
	the growth in relative productivity becomes trivial.
	
	Multi-UAV collaboration is a potentially promising choice to increase the transportation capacity.
	On top of this, additional task redundancy, lower cost, and robustness to vehicle failure may also be provided \cite{mellinger2013cooperative, dongxiWang, villa2020survey,geng2020cooperative}.
	Generally, the UAV-payload connection types include active connections and passive connections \cite{zeng2020differential}.
	Active connection involves equipping the vehicle with a gripper to grasp and hold the payload rigidly \cite{manipulatorCooperation},
	while passive connection refers to suspending the payload through cables \cite{Arcak} or via a universal joint \cite{tagliabue2019robust}.
	The gripper attachment increases the mass and inertia of the system considerably and thereby makes the system respond slowly.
	In contrast, the cable suspension mechanism is more appealing for its low cost and flexible system structure.
	Therefore, we adopt the cable suspension mechanism for collaborative transportation in this research.
	
	The state of the art of control strategies for cable-suspended collaborative transportation can be divided into two groups \cite{qianlongHao}:  payload-based  design and  formation-based design.
	Payload-based design focuses on the trajectory of the payload, e.g., \cite{goodarzi2016stabilization, sreenath2013dynamics, Taeyoung2018}. 
	Although the precise attitude and position control of the payload can be realized,  
	the dynamic information of the payload is required for real-time feedback control, 
	which is hard to obtain in engineering practice.
	 In contrast,  in formation-based design,
	 only the state information of the aerial vehicles is needed.
	 When the vehicle group reaches its destination,
	 the payload is also supposed to reach the target area.
	 The validity and feasibility of such approach has been established via simulation  \cite{SHIRANI2019158} and experiment \cite{bacelar2020board},
	 but the cable forces on the quadrotors are ignored.

	 To implement formation-based robust collaborative transportation,
	 several control algorithms have been developed.
	 A distanced-based formation control algorithm for a team of quadrotors transporting a heavy object is presented in \cite{Marina}, 
	 which measures and resists the acceleration due to disturbances and rope tension using incremental nonlinear dynamic inversion control.
	 In \cite{klausen2018cooperative} and \cite{meissen2017passivity},
	 a passivity-based formation control strategy is proposed with adaptive compensation terms to eliminate the wind disturbance and the cable tension.
	 The energy passivity property of the quadrotors-payload system is established in \cite{Mohammadi},
	 where an adaptive damping term is used to dissipate the energy injected by the sudden perturbations.
	 
	 The studies mentioned above are all designed based on the rigid formation.
	 As a matter of fact,
	 maintaining a fixed formation for payload transportation is not necessary and
	 it is better to employ a flexible formation, which can adapt the vehicles to the complex and uncertain environment and tasks \cite{doakhan4222094cooperative}.  
	Force control-based approaches have been explored for collaborative payload transportation with flexible formation, e.g., force amplification \cite{wang2016force} and contact force regulation \cite{thapa2020cooperative}.
	 The so-called Force-Amplifying N-Robot Transport System (Force-ANTS) control framework is introduced in \cite{wang2016force} to achieve force-coordination among a group of ground robots.
	 The follower robots perceive the leader force by simply measuring the object's motion locally and then reinforce this intention,
	 which makes it possible to cooperatively transport heavy objects of various sizes without any communication network.
	 In \cite{thapa2020cooperative},
	 a new adaptive force-consensus algorithm is proposed to guarantee  the average load distribution among the vehicles using force/acceleration sensors.
	 This work can average the energy consumption among the UAVs and thereby extend the endurance of the entire team.
	 Cooperative manipulation of a cable-suspended payload with two aerial vehicles  is considered in \cite{Tognon},
	 where the role of the internal force is first studied and analyzed in depth.
	 Although simulation results have verified the effectiveness of the methods in \cite{thapa2020cooperative} and \cite{Tognon},
	 these methods cannot be applied directly in practice because model uncertainties are not considered.

	 In this article, we propose a force-coordination control strategy for cable-suspended collaborative transportation. Here the concept of force-coordination means that cable forces between the payload and aerial vehicles converge to the expected values cooperatively, which is believed to  play a fundamental role in more difficult aerial cooperative payload manipulations, such as swinging a payload \cite{swingpayload}.
	 The most critical step for applying force-coordination-based control is the accurate measurement of contact force.
	 In \cite{forceSensors1} and \cite{forceSensors2},  force sensors are installed to measure the cable tensions.
	 However, in addition to the high cost, the force sensor complicates the vehicle's structure and increases the  weight of the whole system,
	 whereas
	 force estimation is  more appropriate  for its low cost and convenience.
	 The existing external force estimation methods, e.g., disturbance observer (DO) \cite{Cuiyangyang}, extended state observer (ESO) \cite{ESO}, and unscented Kalman filter (UKF) \cite{UKF}, can only estimate the equivalent lumped disturbance rather than distinguish different disturbances in the same channel.
	 The cable tension is always coupled with multiple disturbances, like thrust uncertainty, wind force,  and  mass center offset, in the acceleration channel.
	 Therefore, it is not a straightforward task to estimate the cable force precisely. 
	 An attempt is made to estimate the contact force of rigidly connected payload  for admittance control in \cite{tagliabue2019robust}.
	To avoid the undesirable offset in the estimated force caused by wind and model uncertainties,
	 a Finite State Machine is employed to monitor the magnitude of the force and decides whether to reject or utilize the estimate for trajectory generation.
	 This strategy seems quite fascinating and practical for its robustness to disturbances, 
	 but essentially it only evaluates the quality of estimation and does not improve the force estimation accuracy.
	 
	This paper studies the collaborative transportation system composed of two aerial vehicles carrying a cable-suspended long pipe.
	The lengths of the cables are different and unknown, 
	so that we can treat this system as a heterogeneous coordination system.
	For quadrotor dynamics, among the multiple disturbances mixed with the cable force, thrust uncertainty is the primary one, which is the synthesis of uncertainties in the whole propulsion system. Uncertainties in the propulsion system include aerodynamic uncertainties and hardware uncertainties. Here aerodynamic uncertainties refer to the thrust coefficients, which are consistent for the same blades. However, hardware uncertainties vary from one-to-one, including motor degradation, blade damage, battery wear, electronic speed controller efficiency loss, and so on. Therefore, to acquire an accurate cable force estimate, it is necessary to get rid of the thrust uncertainty.
	The main contributions  are summarized in the following aspects:
	\begin{enumerate}
		\item A force-coordination control scheme is proposed for the collaborative transportation system. Different from the position-coordination control methods \cite{klausen2018cooperative, meissen2017passivity, Mohammadi}, cable forces instead of positions are used to regulate the formation and motion of the collaborative vehicles. When applied to the heterogeneous quadrotors with different cable lengths, the pipe can be aligned parallel to the ground under the force-consensus condition, which also means the equal share of the payload mass.
		\item A sensorless lumped disturbance separation and estimation strategy based on disturbance observer (DO) is developed. Here DO is used to estimate the lumped force disturbance for the nominal dynamic model of the quadrotor. Under the quasi-static condition, a separation mechanism is first introduced to separate the significant thrust uncertainty from the lumped force disturbance estimate, so that more precise estimate of the cable force can be obtained for force-coordination  control.
		\item Real-world flight tests are carried out to verify the effectiveness of the proposed force-coordination control algorithm. 
		To our best knowledge, such force-coordination test without force sensor has not  been reported in previous studies.
		The test results show that the thrust uncertainty separation performs as expected and the pipe is stabilized within the small range of $1^\circ$ to $3^\circ$ near the equilibrium.
		
	\end{enumerate}

\section{Problem Formulation and Notations}\label{section 2}
\subsection{Mathematical Preliminaries}
The special orthogonal group is denoted as
\begin{equation}
	\nonumber
	\mathsf{SO}(3) = \left\{\bm{\mathcal{A}} \in \mathbb{R}^{3 \times 3} | \bm{\mathcal{A}}^\top \bm{\mathcal{A}} = \bm{\mathcal{A}} \bm{\mathcal{A}}^\top = \bm{I}_3, \mathrm{det}(\bm{\mathcal{A}}) = 1\right\}
\end{equation}
where $ \bm{I}_3 \in \mathbb{R}^{3 \times 3}$ is the identity matrix.

The set of $3 \times 3$ skew-symmetric matrix is denoted as
\begin{equation}
	\nonumber
	\mathfrak{so}(3) = \{\bm{\mathcal{B}} \in \mathbb{R}^{3 \times 3} | \bm{\mathcal{B}}^\top = - \bm{\mathcal{B}}\}
\end{equation}
which corresponds to the Lie algebra of $\mathsf{SO}(3)$.

The two-sphere $\mathsf{S}^2$ is the set of all unit vectors in the Euclidean space $\mathbb{R}^3$, i.e.,
\begin{equation}
	\nonumber
	\mathsf{S}^2 = 	\{\bm{h} \in \mathbb{R}^3 | \bm{h}^\top \bm{h} = 1\}.
\end{equation}

The Euclidean norm of  a matrix $\bm{\mathcal{C}} \in \mathbb{R}^{m \times n}$ is defined as
\begin{equation}\nonumber
	\begin{aligned}
		\| \bm{\mathcal{C}} \| = \sqrt{\lambda_{\max} (\bm{\mathcal{C}}^\top \bm{\mathcal{C}})}
	\end{aligned}
\end{equation}
where $\lambda_{\max}$ is denoted as the largest eigenvalue of the matrix.

For vector $\bm{\omega} = [\omega_1, \omega_2, \omega_3]^\top$,  we define the mapping ${(\cdot)}^{\times} : \mathbb{R}^3 \to \mathfrak{so}(3)$ as
\begin{equation}\nonumber
	\begin{aligned}
		\bm{\omega}^{\times} = 
		\begin{bmatrix}
			0 & -\omega_3 & \omega_2 \\
			\omega_3 & 0 & -\omega_1\\
			-\omega_2 & \omega_1 & 0
		\end{bmatrix}
	\end{aligned}
\end{equation}
and the projection mappings $(\cdot)_{xy}: \mathbb{R}^3 \to \mathbb{R}^2$ and $(\cdot)_{z}: \mathbb{R}^3 \to \mathbb{R}$ as
\begin{equation}\nonumber
	\begin{aligned}
		\bm{\omega}_{xy} = \begin{bmatrix}
			\omega_1 & \omega_2
		\end{bmatrix}^\top, \quad
	\bm{\omega}_{z} = \omega_3.
	\end{aligned}
\end{equation}

\subsection{Configuration Description}
\begin{figure}[!hbt]
	\centering
	\includegraphics[width=3.5in]{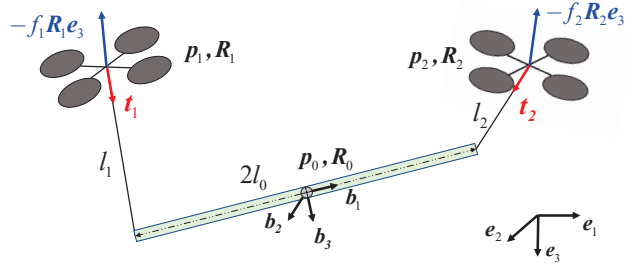}
	\caption{Schematic of the collaborative transportation system.}
	\label{fig_2}
\end{figure}
The quadrotors-payload structure  interconnected by massless cables is shown in Fig. \ref{fig_2},
where it is assumed that the payload is a round pipe and the cables are attached to the center of mass (CoM) of the quadrotors, without inducing additional torque on the quadrotors.
The north-east-down (NED) frame is chosen as the inertial frame $\mathcal{F_I} = \{\bm{e}_1, \bm{e}_2, \bm{e}_3\}$, 
with $\bm{e}_1 = \left[1, 0, 0\right]^{\top}$, $\bm{e}_2 = \left[0, 1, 0\right]^{\top}$ and $\bm{e}_3 = \left[0, 0, 1\right]^{\top}$.
The body-attached frame for the payload is $\mathcal{F}_{\mathcal{B}0} = \{\bm{b}_1, \bm{b}_2, \bm{b}_3\}$, 
with its origin at the CoM of the payload; $\bm{b}_1$ axis points along the pipe  to the  suspension point of the second cable;
 $\bm{b}_3$ axis is perpendicular to $\bm{b}_1$, locates in the vertical plane, and points downward; $\bm{b}_2$ axis completes the right-hand frame.
The half length of the pipe is denoted as $l_0$.
 The  mass distribution of the pipe is assumed to be  uniform,
so that the CoM position of the pipe coincides with its geometric center at $\bm{p}_0 \in \mathbb{R}^3$, and the attitude is described by the rotation matrix $\bm{R}_0 \in \mathsf{SO}(3)$.
In addition, the CoM position and the attitude of the $i$th quadrotor are denoted by $\bm{p}_i \in \mathbb{R}^3$ and $\bm{R}_i \in \mathsf{SO}(3)$, respectively.
The cables are assumed to be taut and the length of the $i$th cable $l_i$ is assumed to be fixed.

\subsection{System Dynamics}
The dynamic model for the payload is derived as \cite{sreenath2013dynamics}:
\begin{equation}\label{payload dynamics}
	\begin{aligned}
		\ddot{\bm{p}}_0 &= g \bm{e}_3 - \frac{1}{m_0} \bm{t}_1 - \frac{1}{m_0} \bm{t}_2\\
		\dot{\bm{R}}_0 &= \bm{R}_0 \bm{\Omega}_0^{\times}\\
		\dot{\bm{\Omega}}_0 &=	\bm{J}_0^{-1} \left[-\bm{\Omega}_0^{\times} \bm{J}_0 \bm{\Omega}_0 + l_0 \bm{e}_1^{\times} \bm{R}_0^\top \left( \bm{t}_2 - \bm{t}_1\right)\right]
	\end{aligned}
\end{equation}
where $g \in \mathbb{R}^{+}$ is the gravity constant; $m_0 \in \mathbb{R}^+$ and $\bm{J}_0 \in \mathbb{R}^{3 \times 3}$ are the mass and the inertia matrix of the payload, respectively; $\bm{t}_i \in \mathbb{R}^3$ is the cable force on the $i$th quadrotor;
$\bm{\Omega}_0 \in \mathbb{R}^3$ is the body angular rate of the payload.

We  consider cable force $\bm{t}_i$ and thrust uncertainty $\Delta f_i \in \mathbb{R}$ as the force disturbances for the quadrotors. 
The actual thrust and the command thrust for the $i$th quadrotor are denoted as $f_i \in \mathbb{R}^{+}$ and $f_{ic} \in \mathbb{R}^{+}$ respectively, and satisfy the relation $f_{i} = f_{ic} + \Delta f_i$.
Therefore, the dynamic model for the $i$th quadrotors is expressed as
\begin{subequations}
	\begin{equation}
	\label{system model}
	\begin{aligned}
		\ddot{\bm{p}}_i =  g \bm{e}_3 - \frac{f_{ic}+\Delta f_i }{m_i} \bm{R}_i \bm{e}_3  + \frac{1}{m_i} \bm{t}_i 
	\end{aligned}
\end{equation}
\begin{equation}
	\label{rotational dynamics}
	\begin{aligned}
		\dot{\bm{R}}_i = \bm{R}_i \bm{\Omega}_i^{\times}
	\end{aligned}
\end{equation}
\begin{equation}
	\begin{aligned}
		\dot{\bm{\Omega}}_i = \bm{J}_i^{-1} \left(-\bm{\Omega}_i^{\times} \bm{J}_i \bm{\Omega}_i + \bm{\tau}_i\right)
	\end{aligned}
\end{equation}
\end{subequations}
 where $m_i \in \mathbb{R}^+$ and $\bm{J}_i \in \mathbb{R}^{3 \times 3}$ are the mass and the inertia matrix of the $i$th quadrotor, respectively;
 $\bm{\Omega}_i \in \mathbb{R}^3$ is the body angular rate of the $i$th quadrotor;
 $\bm{\tau}_i \in \mathbb{R}^3$ is the torque input for the $i$th quadrotor.
  
   The inner-loop dynamics control is assumed to be sufficiently fast and accurate to track the desired attitude command. Thus, one can consider the outer-loop and the inner-loop separately, similar to \cite{klausen2018cooperative}.
   In this manner,
  define the control input $\bm{u}_i \in \mathbb{R}^3$ for the outer-loop dynamics \eqref{system model} as
    \begin{equation}\label{control input def}
  	\begin{aligned}
  		\bm{u}_i \triangleq - \frac{f_{ic}}{m_i} \bm{R}_i \bm{e}_3, \quad i = 1, 2
  	\end{aligned}
  \end{equation}
   which can be regarded as the desired translational  acceleration for the $i$th quadrotor to be designed later.
	In addition, the cable force $\bm{t}_i$ and the thrust uncertainty $\Delta f_i$ are treated as the lumped disturbance $\bm{d}_{i} = [d_{ix}, d_{iy}, d_{iz}]^\top \in \mathbb{R}^3$, i.e.,
	\begin{equation}\label{lumped disturbance}
		\bm{d}_i  \triangleq - \frac{\Delta f_i} {m_i} \bm{R}_{i}\bm{e}_3 + \frac{1 }{m_i} \bm{t}_{i}, \quad i = 1, 2.
	\end{equation}
 The translational model \eqref{system model} is rewritten as:
   \begin{equation}
   	\label{control model}
   	\begin{aligned}
   		\ddot{\bm{p}}_i &=  g \bm{e}_3 + \bm{u}_i + \bm{d}_i, \quad i = 1, 2.
   	\end{aligned}
   \end{equation}
   From the definition \eqref{control input def}, we can obtain the magnitude of the command thrust for the $i$th quadrotor as
   \begin{equation}
   			f_{ic} = - m_i \bm{u}_i^\top \bm{R}_i \bm{e}_3, \quad i = 1, 2.
   \end{equation}

   \subsection{Inner loop Control}
We adopt the method in \cite{pucci2015} to design an inner-loop attitude controller for the single quadrotor
to guarantee that the direction of the actual thrust converges to the direction of the desired acceleration $\bm{u}_i$ exponentially.

For the desired translational acceleration $\bm{u}_i$,
the direction is given by
\begin{equation}
	\bm{h}_{i d}  = \frac{\bm{u}_i}{\| \bm{u}_i \|} \in \mathsf{S}^2
\end{equation}
noting that $\|\bm{u}_i \| \neq 0$ for quadrotors.
Then the desired body angular rate $\bm{\Omega}_{i \bm{d}} \in \mathbb{R}^3$  for the rotational dynamics  \eqref{rotational dynamics} is designed as
\begin{equation}
	\bm{\Omega}_{i \bm{d}} = \left(k_z + \frac{\dot{\gamma}_i}{\gamma_i} \right) \bm{e}_3^{\times} \bm{R}_i^\top \bm{h}_{i \bm{d}}  + \left(\bm{I}_3 - \bm{e}_3 \bm{e}_3^\top\right) \bm{R}_i^\top \bm{h}_{i \bm{d}}^{\times} \dot{\bm{h}}_{i \bm{d}}
\end{equation}
where $\gamma_i = \sqrt{c + \| m_i \bm{u}_i\|^2 } \in \mathbb{R}^+$,
$c \in \mathbb{R}^{+}$ is a small constant,
and $k_z \in \mathbb{R}^+$  is the control gain for attitude tracking.

The torque input for angular rate tracking is then given by
\begin{equation}
	\begin{aligned}
		\bm{\tau}_i &= \bm{\Omega}_i^{\times} \bm{J}_i \bm{\Omega}_i - \bm{K}_{\bm{\Omega}} \left(\bm{\Omega}_i - \bm{\Omega}_{i \bm{d}}\right)
	\end{aligned}
\end{equation}
where $\bm{K}_{\bm{\Omega}} \in \mathbb{R}^{3 \times 3}$ is the control gain.
Detailed proof for exponential convergence can refer to \cite{pucci2015}.

\section{Position-Coordination Control}\label{section 3}
To transport the payload to the desired position, 
we first employ the  leader-follower formation control structure in \cite{Leader_Follower}, where  quadrotor 1 is the leader and knows the desired position, and quadrotor 2 is the follower to keep the formation. 

Denote the desired trajectory of  quadrotor 1 as $\bm{p}_{1 d} \left(t\right) = \left[p_{1dx}(t), p_{1dy}(t), p_{1dz}(t)\right]^\top \in \mathbb{R}^3$ and the desired relative position as $\bm{p}_{12d} (t)=  \left[p_{12dx}(t), p_{12dy}(t), p_{12dz}(t)\right]^\top \in \mathbb{R}^3$
which are generated by the upper-level motion planning algorithm.
Throughout this article, we assume that $\dot{\bm{p}}_{12 d}(t) = \ddot{\bm{p}}_{12d}(t) = \bm{0}_{3 \times 1}$,
i.e., the desired spatial formation is time-invariant.
In the rest of the paper, we often do not explicitly write the dependence
on $t$ of the variables for notation convenience.

\subsection{Control Objective}
Here the control objective is to develop control laws for the two vehicles to achieve the following behaviors:
\begin{itemize}
	\item Quadrotor 1 achieves the desired trajectory 
	\begin{equation}
\bm{p}_1 (t) - \bm{p}_{1d} (t) \to \bm{0}_{3 \times 1}.
	\end{equation} 
	\item The two quadrotors keep the spatial formation
	\begin{equation}
		\bm{p}_1 (t)- \bm{p}_2 (t)\to \bm{p}_{12d}.
	\end{equation}
\end{itemize}
\subsection{Lumped Disturbance Estimation}
	\begin{assumption}\label{assumption 1}
			There exists an unknown positive constant $\bar{d}$ such that
		the time derivative of $\bm{d}_i$ in \eqref{control model} is bounded, i.e.,
			$\|\dot{\bm{d}}_{i}(t)\| \le \bar{d}$.
	\end{assumption}
  \begin{remark}
  	The change rates of attitude angles, thrust uncertainty, and cable force
  	 are physically limited \cite{Chen2015disturbance}, 
  	which reveals that Assumption 1 coincides with the common
  	practice.
  \end{remark}
To estimate the lumped disturbance $\bm{d}_i (t)$ in \eqref{control model},
the disturbance observer is designed as \cite{Chen2015disturbance}
\begin{equation}
	\label{disturbance estimate}
	\begin{aligned}
		\dot{\bm{z}}_i &= -\iota \left(\bm{z}_i + \iota \dot{\bm{p}}_i + g \bm{e}_3 + \bm{u}_i \right)\\
		\hat{\bm{d}}_i &= \bm{z}_i + \iota \dot{\bm{p}}_i
	\end{aligned}
\end{equation}
where $\bm{z}_i \in \mathbb{R}^3$ is the auxiliary state,
$\iota \in \mathbb{R}^{+}$ is the observer gain,
and $\hat{\bm{d}}_i = [\hat{d}_{ix}, \hat{d}_{iy}, \hat{d}_{iz}]^\top \in \mathbb{R}^3$ denotes the disturbance estimate.
Define the disturbance estimation error as 
\begin{equation}\label{estimationerr}
\tilde{\bm{d}}_i = \hat{\bm{d}}_i - \bm{d}_i \in \mathbb{R}^3
\end{equation}
 and the error dynamics can be derived as
\begin{equation}\label{estimation error dynamics}
	\begin{aligned}
		\dot{\tilde{\bm{d}}}_i =& - \iota \tilde{\bm{d}}_i - \dot{\bm{d}}_i.
	\end{aligned}
\end{equation}
According to Assumption \ref{assumption 1},
the boundedness of $\tilde{\bm{d}}_i(t)$ can be established \cite{Chen2015disturbance} with  $\bar{\tilde{d}}$ being an unknown positive constant
\begin{equation}\label{disturbance bounded}
	\begin{aligned}
		\|\tilde{\bm{d}}_i (t) \| \le \bar{\tilde{d}}.
	\end{aligned}
\end{equation}

\subsection{Position-coordination-based Rigid Formation Control}

The leader-follower controllers for the two quadrotors are designed as 
\begin{equation}
	\label{formation controller1}
\begin{aligned}
	\bm{u}_1 =& -k_1 \tilde{\bm{p}}_{12}  - k_2 \dot{\tilde{\bm{p}}}_{12} + \bm{\pi}_1
	- g \bm{e}_3 - \hat{\bm{d}}_1\\
	\bm{u}_2 =& \quad k_1 \tilde{\bm{p}}_{12}  + k_2 \dot{\tilde{\bm{p}}}_{12} 
	- g \bm{e}_3 - \hat{\bm{d}}_2\\
\end{aligned}
\end{equation}
where  $k_1 \in \mathbb{R}^{+}$ and $k_2 \in \mathbb{R}^{+}$ are the controller gains for formation keeping;
$\tilde{\bm{p}}_{12} \in \mathbb{R}^{3}$ is the formation error 
\begin{equation}\label{formation error}
	\begin{aligned}
		\tilde{\bm{p}}_{12} &= \bm{p}_1 - \bm{p}_2 - \bm{p}_{12d}
	\end{aligned}
\end{equation}
and $\bm{\pi}_1 \in \mathbb{R}^{3}$  is the external input
\begin{equation} \label{external input}
	\begin{aligned}
	\bm{\pi}_1 &= -k_3 (\bm{p}_1 - \bm{p}_{1d}) - k_4(\dot{\bm{p}}_1 - \dot{\bm{p}}_{1d})
	\end{aligned}
\end{equation}
with $k_3 \in \mathbb{R}^{+}$ and $k_4 \in \mathbb{R}^{+}$ being the controller gains for reference trajectory tracking.

 \subsection{Stability Analysis}
 Denote the position tracking error and velocity tracking error of quadrotor 1 by $\tilde{\bm{p}}_1$ and $\dot{\tilde{\bm{p}}}_1$ respectively,
 i.e., $\tilde{\bm{p}}_1 = \bm{p}_1 - \bm{p}_{1d}$ and $\dot{\tilde{\bm{p}}}_1 = \dot{\bm{p}}_1 - \dot{\bm{p}}_{1d}$.
 \begin{theorem}
 	Consider two quadrotors carrying a suspended payload, modeled as \eqref{control model}, with the control laws \eqref{formation controller1},
 	and the disturbance estimate updated as \eqref{disturbance estimate}.
 	Suppose that Assumption \ref{assumption 1} holds. 
	 	Then the signals of the overall closed-loop system $(\tilde{\bm{p}}_1, \dot{\tilde{\bm{p}}}_1, \tilde{\bm{p}}_{12}, \dot{\tilde{\bm{p}}}_{12}, \tilde{\bm{d}}_1, \tilde{\bm{d}}_2)$ are uniformly ultimately bounded.
 \end{theorem}
\begin{proof}
	Since the boundedness of estimation error $\tilde{\bm{d}}_{i}$ is established in \eqref{disturbance bounded} and  independent of the boundedness of the tracking error and formation error,
	we only focus on the stability of signals $(\tilde{\bm{p}}_1, \dot{\tilde{\bm{p}}}_1, \tilde{\bm{p}}_{12}, \dot{\tilde{\bm{p}}}_{12})$. 
	
Substituting \eqref{formation controller1}  into \eqref{control model} yields 
\begin{equation}
	\begin{aligned}
			\ddot{\bm{p}}_1 &= -k_1  \tilde{\bm{p}}_{12} - k_2 \dot{\tilde{\bm{p}}}_{12} - k_3 \tilde{\bm{p}}_1 - k_4 \dot{\tilde{\bm{p}}}_1- \tilde{\bm{d}}_1\\
		\ddot{\bm{p}}_2 &=  \quad k_1  \tilde{\bm{p}}_{12} + k_2 \dot{\tilde{\bm{p}}}_{12} - \tilde{\bm{d}}_2.
	\end{aligned}
\end{equation}
Then the dynamics of tracking error  $\tilde{\bm{p}}_1$ is derived as
\begin{equation}
	\begin{aligned}
			\ddot{\tilde{\bm{p}}}_{1} =   -k_1  \tilde{\bm{p}}_{12} - k_2 \dot{\tilde{\bm{p}}}_{12} - k_3 \tilde{\bm{p}}_1 - k_4 \dot{\tilde{\bm{p}}}_1- \tilde{\bm{d}}_1 - \ddot{\bm{p}}_{1d}
	\end{aligned}
\end{equation}
and the dynamics of formation error $ \tilde{\bm{p}}_{12}$ is computed as
\begin{equation}
	\begin{aligned}
			\ddot{\tilde{\bm{p}}}_{12} &= \ddot{\bm{p}}_1 - \ddot{\bm{p}}_2\\
					&= -2k_1  \tilde{\bm{p}}_{12} - 2 k_2 \dot{\tilde{\bm{p}}}_{12} - k_3 \tilde{\bm{p}}_1 - k_4 \dot{\tilde{\bm{p}}}_1- \tilde{\bm{d}}_1 + \tilde{\bm{d}}_2.
	\end{aligned}
\end{equation}
Define the vector $\bm{\zeta} = [\tilde{\bm{p}}_1^\top,  \tilde{\bm{p}}_{12}^\top, \dot{\tilde{\bm{p}}}_1^\top, \dot{\tilde{\bm{p}}}_{12}^\top]^\top \in \mathbb{R}^{12}$ and the vector $\tilde{\bm{d}} = [\tilde{\bm{d}}_1^\top, \tilde{\bm{d}}_2^\top]^\top \in \mathbb{R}^6$,
yielding
\begin{equation}
		\label{zeta}
	\begin{aligned}
		\dot{\bm{\zeta}} = \bm{A} \bm{\zeta} + \bm{B}_1 \tilde{\bm{d}} + \bm{B}_2 \ddot{\bm{p}}_{1d}
	\end{aligned}
\end{equation}
where 
\begin{equation}
	\begin{aligned}
		\bm{A} &= 
		\begin{bmatrix}
			\bm{0}_{3 \times 3} & \bm{0}_{3 \times 3} & \bm{I}_{3} &\bm{0}_{3 \times 3}\\
			 \bm{0}_{3 \times 3} &  \bm{0}_{3 \times 3}  &  \bm{0}_{3 \times 3} & \bm{I}_{3}\\
			 -k_3 \bm{I}_{3} & -k_1 \bm{I}_{3} & -k_4\bm{I}_{3} & -k_2 \bm{I}_{3}\\
 			-k_3 \bm{I}_{3} & -2 k_1 \bm{I}_{3} & -k_4 \bm{I}_{3} & -2 k_2 \bm{I}_{3}
		\end{bmatrix} \in \mathbb{R}^{12 \times 12}\\
\bm{B}_1 &= 
\begin{bmatrix}
	\bm{0}_{3 \times 3} & \bm{0}_{3 \times 3} &- \bm{I}_{3} & - \bm{I}_{3}\\
	\bm{0}_{3 \times 3} & \bm{0}_{3 \times 3} & \bm{0}_{3 \times 3} & - \bm{I}_{3}
\end{bmatrix}^\top \in \mathbb{R}^{12 \times 6}\\
\bm{B}_2 &=
\begin{bmatrix}
	\bm{0}_{3 \times 3} & \bm{0}_{3 \times 3} &- \bm{I}_{3} & \bm{0}_{3 \times 3}
\end{bmatrix}^\top \in \mathbb{R}^{12 \times 3}.
	\end{aligned}
\end{equation}

The characteristic polynomial of the matrix $\bm{A}$ is
\begin{equation}
	\begin{aligned}
		p(s) &=  s^4 + (2 k_2 + k_4) s^3 + (2 k_1 + k_3 + k_2 k_4) s^2 \\
		&\quad + (k_1 k_4+ k_2 k_3) s + k_1 k_3.
	\end{aligned}
\end{equation}
The controller gains $k_1$, $k_2$, $k_3$, and $k_4$ are chosen
 according to Routh-Hurwitz stability criterion,
 so that all roots of the characteristic polynomial are  in  the negative half plane,
implying the negative definiteness of the matrix $\bm{A}$.
According to the Lyapunov equation,
given any $\bm{Q} > 0$,
there exists a unique $\bm{P} > 0$ satisfying
$\bm{P} \bm{A} + \bm{A}^\top \bm{P} = -\bm{Q}$.

Next, define a Lyapunov function as
\begin{equation}
	V = \bm{\zeta}^\top \bm{P} \bm{\zeta}.
\end{equation}
Then the derivative of $V$ is 
\begin{equation}
	\begin{aligned}
	\dot{V} &= -\bm{\zeta}^\top \bm{Q} \bm{\zeta} + 2 \bm{\zeta}^\top \bm{P} \bm{B}_1 \tilde{\bm{d}} + 2\bm{\zeta}^\top \bm{P} \bm{B}_2 \ddot{\bm{p}}_{1d} \\
	&\le -\bm{\zeta}^\top \bm{Q} \bm{\zeta} + 2 \|\bm{\zeta}\| \|\bm{P}\| \|\bm{B}_1\| \| \tilde{\bm{d}}\| +  2 \|\bm{\zeta}\| \|\bm{P}\| \|\bm{B}_2\| \| \ddot{\bm{p}}_{1d}\| \\
	&\le -\bm{\zeta}^\top \bm{Q} \bm{\zeta} + (\sqrt{10}+\sqrt{2}) \bar{\tilde{d}} \|\bm{P}\|  \|\bm{\zeta}\| + 2 \| \ddot{\bm{p}}_{1d}\| \|\bm{P}\|   \|\bm{\zeta}\| 
		\end{aligned}
\end{equation}
	where $\|\tilde{\bm{d}}\| \le \sqrt{2} \bar{\tilde{d}}$, $\| \bm{B}_1\| = \frac{\sqrt{5}+1}{2}$, and $\| \bm{B}_2 \| = 1$;
	$\ddot{\bm{p}}_{1d}$ is always bounded according to the preset trajectory planning.
	
Finally, from Lyapunov boundedness theory \cite{khalil2002nonlinear}, $\bm{\zeta}$ is uniformly ultimately bounded,
implying the boundedness of $(\tilde{\bm{p}}_1, \dot{\tilde{\bm{p}}}_1,  \tilde{\bm{p}}_{12}, \dot{\tilde{\bm{p}}}_{12})$.
\end{proof}

 \section{Force-Coordination Control} \label{section 4}
Different from the position-coordination control law \eqref{formation controller1},
 we further propose flexible formation control laws based on force-coordination.
 In this study,
 the two quadrotors share the weight of the load equally in force-consensus condition while keeping the formation in the horizontal plane.
 \subsection{Control Objective}
 In this section, a force-coordination term is incorporated into the vehicle control laws to achieve the following behaviors:
 \begin{itemize}
 	\item Quadrotor 1 achieves the desired position
 	\begin{equation}
	\bm{p}_1 (t)- \bm{p}_{1d} \to \bm{0}_{3 \times 1}. 
 	\end{equation}  
 Here the desired  position of quadrotor 1 is assumed fixed in this section, i.e., $\dot{\bm{p}}_{1d} = \ddot{\bm{p}}_{1d} = \bm{0}_{3 \times 1}$.
 	\item The two quadrotors keep the formation in the horizontal plane
 	\begin{equation}
 	\bm{p}_{1xy}(t) - \bm{p}_{2xy}(t) \to \bm{p}_{12dxy}.
 	\end{equation}
 	\item The two quadrotors achieve the equal cable forces,
 	 without knowing  the cable lengths
 	\begin{equation}
		t_{2z} - t_{1z} \to 0.
 	\end{equation}
 \end{itemize}

\subsection{Equilibrium Analysis}\label{Equilibrium Analysis}
This subsection aims to  analyze the equilibria of the payload corrsponding to the equal cable forces in the vertical direction, and  give an explanation of the role of the load internal force.
\begin{proposition}
	Consider the system composed of two quadrotors and a pipe-like suspended payload described in  Fig. \ref{fig_2}.
	Under the conditions that force-consensus in the vertical direction and the desired horizontal formation are achieved,
	if the internal force is non-zero,
	then the equilibrium configuration of the pipe is parallel to the ground.
\end{proposition}
\begin{proof}
	When the pipe is at its equilibrium ($\ddot{\bm{p}}_0 = \bm{0}_{3 \times 1}$ and $\bm{\Omega}_0 = \bm{0}_{3 \times 1}$),
	 the following equations are derived from the dynamics \eqref{payload dynamics} as
	\begin{subequations}
		\begin{equation}\label{gravity}
			\bm{t}_1 + \bm{t}_2 -  m_0 g \bm{e}_3 = \bm{0}_{3 \times 1}
		\end{equation}
		\begin{equation}
			\label{equilibrium}
			\begin{aligned}
				\bm{e}_1^{\times} \bm{R}_0^\top \left(\bm{t}_2 -  \bm{t}_1\right) = \bm{0}_{3 \times 1}.
			\end{aligned}
		\end{equation}
	\end{subequations} 
	Solving \eqref{equilibrium}  yields
	\begin{equation}\label{cable force error}
		\bm{t}_1 -  \bm{t}_2 = t_0 \bm{R}_0 \bm{e}_1 
	\end{equation}
	where $t_0 \in \mathbb{R}$ is the so-called load internal force \cite{Tognon}.
	
	When the quadrotors and the payload are at the stable static equilibrium,
	the whole system can be modeled as a four-bar-linkage in the plane \cite{michael2009kinematics}.
	Without loss of generality,
	it is assumed that the whole system stays in the XZ  plane of NED frame.
	If the force-consensus condition in the vertical direction is achieved, i.e., $t_{1z} =  t_{2z}$,
	it can be inferred either
	\begin{equation}
		\label{t0=0}
		t_0 = 0
	\end{equation}
	or
	\begin{equation}
		\label{parallel}
		\left\{\begin{aligned}
			t_0 &\neq 0\\
			\bm{R}_0 \bm{e}_1 &= \pm \bm{e}_1.
		\end{aligned}
		\right.
	\end{equation}
	Here the expression \eqref{parallel} corresponds to the situation where the pipe is parallel to the ground. The proof is completed.
	\end{proof}

	Although both conditions \eqref{t0=0} and \eqref{parallel} can achieve force-consensus,
	under condition \eqref{t0=0} the equilibrium attitude for the pipe is not unique.
	To be specific, substituting $t_0 = 0$ into \eqref{cable force error}  and combining \eqref{gravity} yields $\bm{t}_1 = \bm{t}_2 = -\frac{m_0 g}{2}\bm{e}_3$,
	which means when the two cable forces are aligned with the gravity direction, the pipe can theoretically keep stable at any attitude. Obviously, the equilibrium configurations shown in Fig. \ref{fig_4} are not desirable. 

	\begin{figure}[!hbt]
		\centering
		\includegraphics[width=3.5in]{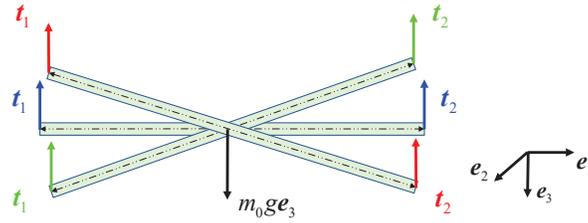}
		\caption{Multiple equilibria for $t_0 = 0$.}
		\label{fig_4}
	\end{figure}

	The internal force $t_0 < 0$ corresponds to the situation that the two quadrotors move closer to the middle of the pipe,
	which is quite dangerous due to the risk of drone collision.
	Therefore, we choose $t_0 > 0$ as the desired internal force shown in Fig. \ref{fig_5}, and the uniqueness of $t_0$ is guaranteed by the horizontal formation setting. 
	\begin{figure}[!hbt]
	\centering
	\includegraphics[width=3.5in]{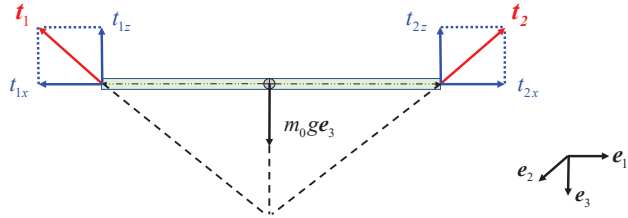}
	\caption{Unique equilibrium for $t_0 > 0$.}
	\label{fig_5}
\end{figure}
\begin{remark}
	In Fig. \ref{fig_5}, since the payload is at its equilibrium, all the forces on the payload intersect at a common point (planar pencil \cite{seo2016theory}),
	constituting three-component force balance.
\end{remark}
 
 \subsection{Force-coordination Formation Control Framework}
 If the cable lengths are known,
 we can configure out the desired trajectories of the quadrotors to keep the pipe parallel to the ground.
 However, without knowing the cable lengths,
 the average load distribution cannot be achieved by the preset configuration planning. 
 Therefore, a force-coordination formation control framework is proposed to overcome this limitation.
 The controller structure is shown in Fig. \ref{fig_6}.
 \begin{figure}[!hbt]
 	\centering
 	\includegraphics[width=3.5in]{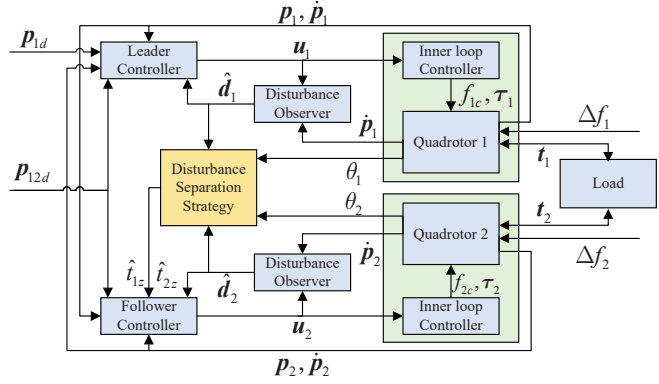}
 	\caption{Controller structure for force-coordination control.}
 	\label{fig_6}
 \end{figure}
 
 For the two quadrotors,
 the controllers  can be decomposed into two parts: 
 one part is to keep the horizontal formation and the other is to achieve force-consensus in the vertical direction.
 
 \subsubsection{Formation keeping in the horizontal plane}
 \
 \newline
 \indent The leader-follower control laws for keeping the formation  in the horizontal plane are designed based on \eqref{formation controller1} as
 \begin{equation}\label{horizontal controller}
 \begin{aligned}
 	\bm{u}_{1 xy} &=  \bm{G} ( -k_1 \tilde{\bm{p}}_{12} - k_2 \dot{\tilde{\bm{p}}}_{12} + \bm{\pi}_1
 	- g \bm{e}_3 - \hat{\bm{d}}_1 )\\
 	 	\bm{u}_{2 xy} &=  \bm{G} ( \quad k_1 \tilde{\bm{p}}_{12}  + k_2 \dot{\tilde{\bm{p}}}_{12} 
 	- g \bm{e}_3 - \hat{\bm{d}}_2)\\
 \end{aligned}
 \end{equation}
where $\bm{u}_{ixy}$ is the projection of the controller $\bm{u}_i$ in \eqref{formation controller1} into the horizontal plane, with
$\bm{G} = [\bm{e}_1, \bm{e}_2 ]^\top \in \mathbb{R}^{2 \times 3}$.
\subsubsection{Force-consensus in the vertical direction}
\
\newline
\indent According to the definition \eqref{lumped disturbance}, the vertical force-consensus error $t_{2z} - t_{1z}$ can be calculated as
\begin{equation}\label{force consensus error}
\begin{aligned}
	t_{2z} - t_{1z} &= \bm{e}_3^\top (m_2 \bm{d}_2 - m_1 \bm{d}_1 + \bm{\Xi})
\end{aligned}
\end{equation}
where  $\bm{\Xi} = \Delta f_2 \bm{R}_2 \bm{e}_3  - \Delta f_1 \bm{R}_1 \bm{e}_3 \in \mathbb{R}^3$ is the difference between the thrust uncertainties of two quadrotors.

Using the lumped disturbance estimates, the estimate for the vertical force-consensus error can be expressed as
\begin{equation}\label{estimate error}
	\begin{aligned}
		\hat{t}_{2z} - \hat{t}_{1z} =  \bm{e}_3^\top (m_2 \hat{\bm{d}}_2 - m_1 \hat{\bm{d}}_1 + \hat{\bm{\Xi}})
	\end{aligned}
\end{equation}
	where $\hat{t}_{iz} \in \mathbb{R}$ is the estimate of the $i$th cable force in the vertical direction and
	 $\hat{\bm{\Xi}} = [\hat{\Xi}_x, \hat{\Xi}_y, \hat{\Xi}_z] \in \mathbb{R}^3$ is the estimate of $\bm{\Xi}$.
	
	For quadrotor 1, the controller in the vertical direction is designed  only for height maintenance as below
	\begin{equation}\label{tracking controller1}
		\begin{aligned}
				u_{1z} = - k_3 (p_{1z} - p_{1dz}) - k_4 \dot{p}_{1z}  - g - \hat{d}_{1z}.
		\end{aligned}
	\end{equation}

For quadrotor 2, the controller in the vertical direction is designed to achieve vertical force consensus between two quadrotors as 
\begin{equation}	\label{vertical controller}
	\begin{aligned}
		u_{2z} =& -k_4 \dot{p}_{2z}  -g -\hat{d}_{2z} + k_f (\hat{t}_{2z} - \hat{t}_{1z})\\
	=&
		-k_4 \dot{p}_{2z}  -g -\hat{d}_{2z} + k_f (m_2 \hat{d}_{2z} - m_1 \hat{d}_{1z} + \hat{\Xi}_z )
		\end{aligned}
\end{equation}
where 
$k_f \in \mathbb{R}^{+}$ is the force-consensus gain.

In summary,
the whole leader-follower controller for the $i$th quadrotor is expressed as
\begin{equation}
\label{force controller 2}
	\bm{u}_i = 
	\begin{bmatrix}
		\bm{u}_{i xy}^\top &	u_{i z}
	\end{bmatrix}^\top
\end{equation}
where $\hat{\Xi}_z$ in $u_{2z}$ cannot be obtained directly from disturbance estimation,
and the detailed estimation method for $\hat{\Xi}_z$ will be developed  in the next subsection.  
\begin{remark}
	More complex manipulations of the pipe beyond averaging load distribution through force-consensus can be realized using the proposed force-coordination control framework by pre-planned force trajectories for the vehicles, which is similar to the method in  \cite{Tognon}. 
\end{remark}

\subsection{Disturbance Separation under Quasi-static Condition}
Quadrotors are underactuated systems that need to rotate to adjust their thrust directions.
Therefore, we can introduce the horizontal force balance equations under the quasi-static condition to obtain $\Delta \hat{f}_i$, and then separate it from the lumped disturbance estimate $\hat{\bm{d}}_i$ for the precise estimate of the cable force $\bm{t}_{i}$.
According to equation \eqref{estimate error},
the better the estimation for $\hat{\Xi}_z$,
the smaller the force-consensus error.
The quasi-static condition refers to the situation that the velocity of the payload changes slowly, i.e.,
\begin{equation}
	\label{balance}
	\begin{aligned}
		\ddot{\bm{p}}_0 \approx \bm{0}_{3 \times 1} 
	\end{aligned}
\end{equation}
which is reasonable during the transportation process.

Without loss of generality,
the quadrotors-payload structure is assumed to stay in the XZ plane of NED frame under  the quasi-static condition in this scenario.
 Substituting $\ddot{\bm{p}}_0 = \bm{0}_{3\times 1}$ into the payload dynamics \eqref{payload dynamics} yields  
\begin{equation}
	\label{XZ balance}
	\begin{aligned}
		t_{1x} + t_{2x} &= 0\\
		t_{1z} + t_{2z} &= m_0 g.
	\end{aligned}
\end{equation}

To be explicit,
according the definitions in \eqref{lumped disturbance} and \eqref{estimationerr}, the disturbance estimate $\hat{\bm{d}}_i$ satisfies the following relation:
\begin{equation}
	\label{quasi-static condition}
	\begin{aligned}
		\hat{\bm{d}}_i &= \bm{d}_i + \tilde{\bm{d}}_i\\
		&= -\frac{\Delta f_i}{m_i} \bm{R}_i \bm{e}_3 + \frac{1}{m_i} \bm{t}_i + \tilde{\bm{d}}_i.
	\end{aligned}
\end{equation}
The attitude of the $i$th quadrotor can be represented by the Euler angles $(\phi_i, \theta_i, \psi_i)$, and the rotation matrix $\bm{R}_i$ can be expressed as 
\begin{equation}
	\begin{aligned}
	\bm{R}_i
	&= \bm{R}_z(0) \bm{R}_y(\theta_i) \bm{R}_x(\phi_i)\\
	&= \begin{bmatrix}
		\cos \theta_i & \sin \phi_i \sin \theta_i & \cos \phi_i \sin \theta_i\\
		0 & \cos \phi_i & - \sin \phi_i\\
		-\sin \theta_i & \sin \phi_i \cos \theta_i  & \cos \phi_i \cos \theta_i
	\end{bmatrix}
	\end{aligned}
\end{equation}
where $\psi_i$ is set as zero.
Expanding \eqref{quasi-static condition} yields
   \begin{equation}
   	\label{balance condition}
   	\begin{aligned}
   		\begin{bmatrix}
   			\hat{d}_{i x} \\
   			\hat{d}_{i z}
   		\end{bmatrix}
   	=
   	\frac{1}{m_i}
   	\begin{bmatrix}
   		-\Delta f_i \sin \theta_i \cos \phi_i + t_{i x} + m_i \tilde{d}_{ix}\\
   		-\Delta f_i \cos \theta_i \cos \phi_i + t_{i z} + m_i \tilde{d}_{iz}
   	\end{bmatrix}.
   	\end{aligned}
   \end{equation}
For  notation simplicity, the following substitutions are adopted
\begin{equation}
	\begin{aligned}
		\hat{\Delta}_x &= m_1 \hat{d}_{1x} + m_2 \hat{d}_{2x}\\
		\hat{\Delta}_z &= m_1 \hat{d}_{1z} + m_2 \hat{d}_{2z} - m_0 g\\
		\tilde{\Delta}_x &= m_1 \tilde{d}_{1x} + m_2 \tilde{d}_{2x}\\
			\tilde{\Delta}_z &= m_1 \tilde{d}_{1z} + m_2 \tilde{d}_{2z}
	\end{aligned}
\end{equation}
where $\hat{\Delta}_x$ and $\hat{\Delta}_z$ denote the estimate results,
$\tilde{\Delta}_x$ and $\tilde{\Delta}_z$ denote the estimation errors.

Combining \eqref{XZ balance} and \eqref{balance condition},
 the thrust uncertainties under the quasi-static condition can be obtained as
\begin{equation}\label{thrust uncertainties}
	\begin{aligned}
		\Delta f_1 &= \frac{\hat{\Delta}_x -\hat{\Delta}_z \tan \theta_2 -\tilde{\Delta}_x + \tilde{\Delta}_z \tan \theta_2 }{  (\cos \theta_1 \tan \theta_2 -\sin \theta_1) \cos \phi_1}\\
			\Delta f_2 &= \frac{\hat{\Delta}_x - \hat{\Delta}_z \tan \theta_1  - \tilde{\Delta}_x + \tilde{\Delta}_z \tan \theta_1}{(\cos \theta_2 \tan \theta_1 - \sin \theta_2)\cos \phi_2 }
	\end{aligned}
\end{equation}
so that $\Xi_z$ is derived as
\begin{equation}
	\label{inconsistency}
	\begin{aligned}
		\Xi_z = \frac{\left(\hat{\Delta}_z - \tilde{\Delta}_z\right) \left(\tan \theta_1 + \tan \theta_2\right) - 2 \left(\hat{\Delta}_x - \tilde{\Delta}_x\right)}{\tan \theta_2 - \tan \theta_1}.
	\end{aligned}
\end{equation}
Then the estimates for the thrust uncertainties $\Delta f_1$ and $\Delta f_2$ can be obtained by removing the estimation errors $\tilde{\Delta}_x$ and $\tilde{\Delta}_z$ from equation \eqref{thrust uncertainties} as
\begin{equation}
	\begin{aligned}
			\Delta \hat{f}_1 &= \frac{\hat{\Delta}_x -\hat{\Delta}_z \tan \theta_2}{  (\cos \theta_1 \tan \theta_2 -\sin \theta_1 )\cos \phi_1}\\
		\Delta \hat{f}_2 &= \frac{\hat{\Delta}_x - \hat{\Delta}_z \tan \theta_1}{(\cos \theta_2 \tan \theta_1 - \sin \theta_2) \cos \phi_2 }.
	\end{aligned}
\end{equation}
Based on equation \eqref{inconsistency}, the estimate $\hat{\Xi}_z$ and the estimation error   $\tilde{\Xi}_z$ can be expressed separately as
\begin{equation}
	\label{inconsistency estimate}
	\begin{aligned}
		\hat{\Xi}_z = \frac{\left(\tan \theta_1 + \tan \theta_2\right)\hat{\Delta}_z - 2\hat{\Delta}_x}{\tan \theta_2 - \tan \theta_1}
	\end{aligned}
\end{equation}
and
\begin{equation}\label{xi}
	\begin{aligned}
		\tilde{\Xi}_z 
		&= \frac{\left(\tan \theta_1 + \tan \theta_2\right) \tilde{\Delta}_z - 2 \tilde{\Delta}_x}{\tan \theta_2 - \tan \theta_1}.
	\end{aligned}
\end{equation}
The vertical cable force estimate $\hat{t}_{iz}$ can be calculated according to equation \eqref{balance condition} as
\begin{equation}
	\begin{aligned}
		\hat{t}_{iz} = m_i \hat{d}_{iz} + \Delta \hat{f}_i \cos \theta_i \cos \phi_i
	\end{aligned}
\end{equation}
which equals to the actual vertical cable force $t_{iz}$ when the estimation error $\tilde{d}_{iz}$ converges to zero.

Finally, the controller \eqref{vertical controller} for force-consensus in the vertical direction is made feasible under the quasi-static condition.

\begin{remark}
	Assuming that precise disturbance estimation is achieved, i.e.,
	$\tilde{\bm{d}}_i \approx \bm{0}_{3 \times 1}$,
	there are six unknown variables $\Delta f_i$, $t_{ix}$ and $t_{iz}$, $i = 1, 2$, which can be solved from six equations by combining equation \eqref{XZ balance} and equation \eqref{balance condition}.
\end{remark}
\subsection{Stability Analysis}
For the stability of quadrotor 1 in the vertical direction,
by substituting \eqref{tracking controller1} into \eqref{control model}, the tracking error dynamics satisfies
\begin{equation}\label{exponential stability}
	\begin{aligned}
		\ddot{\tilde{p}}_{1z} = -k_4 \dot{\tilde{p}}_{1z} -k_3\tilde{p}_{1z} - \tilde{d}_{1z}
	\end{aligned}
\end{equation}
from which the exponential convergence of $\tilde{p}_{1z}$ can be obtained when the disturbance estimation error $\tilde{d}_{1z} \approx 0$. Therefore, the height of quadrotor 1 can be held.

As is shown in Fig. \ref{fig_7}, the pitch angle of the pipe is denoted as $\theta_0$,
and the desired position of quadrotor 2 is denoted as $p_{2d}$,
corresponding to zero pitch angle of the pipe.
When the height of quadrotor 1 and the desired horizontal relative position are fixed,
$p_{2d}$ is a unique equilibrium according to the analysis in Section \ref{Equilibrium Analysis},
so that the desired vertical position of quadrotor 2 can be seen as fixed,
i.e., $\dot{p}_{2dz} = \ddot{p}_{2dz} = 0$.
\begin{figure}[!hbt]
	\centering
	\includegraphics[width=3.5in]{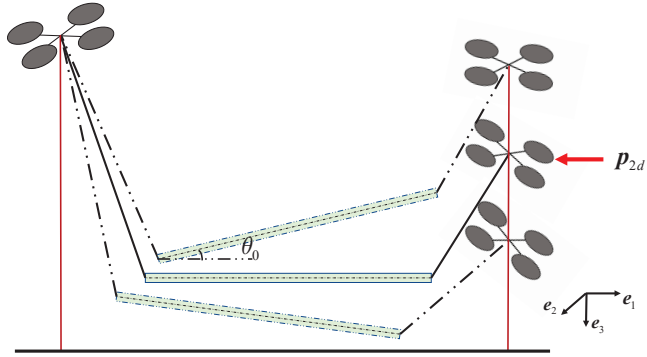}
	\caption{Force-consensus regulation in the vertical direction.}
	\label{fig_7}
\end{figure}

	We require that the desired horizontal  relative position between the two quadrotors satisfies
	\begin{equation}\label{formation setting}
		\| \bm{p}_{12dxy} \| > 2 l_0
	\end{equation} 
	so that the internal force in the pipe is always positive $(t_0 > 0)$.
This formation configuration also guarantees that
	 the pitch angles for the two quadrotors, $\theta_1$ and $\theta_2$, have lower and upper bounds,  i.e.,
	\begin{equation}
		\label{theta assumption}
		\begin{aligned}
			0 <& \underline{\theta}_1 \le \theta_1 \le \overline{\theta}_1 < \frac{\pi}{2}\\
			-\frac{\pi}{2} <& \underline{\theta}_2 \le \theta_2 \le \overline{\theta}_2 < 0.
		\end{aligned}
	\end{equation}

	 Under the quasi-static condition, 
	 the cable forces can be computed  as \cite{Tognon}
	 \begin{equation}
	 	\begin{aligned}
	 		\bm{t}_1 &= 
	 		\left[	\frac{t_0 \cos \theta_0 }{2}, 0, \frac{-t_0 \sin \theta_0 + m_0 g}{2} 
	 		\right]^\top\\
 		\bm{t}_2 &=
 		\left[	-\frac{ t_0 \cos \theta_0}{2}, 0, \frac{t_0 \sin \theta_0 + m_0 g}{2}
 		\right]^\top.
	 	\end{aligned}
	 \end{equation}
The following equation is obtained in the vertical direction
 \begin{equation}
 	\begin{aligned}
 	t_{2z} - t_{1z} = t_0 \sin \theta_0.
 	\end{aligned}
 \end{equation}
Then using \eqref{force consensus error}, \eqref{estimate error}, and \eqref{xi}, the force-consensus error in the vertical direction is estimated as
\begin{equation}\label{force consensus estimate}
	\begin{aligned}
		\hat{t}_{2z} - \hat{t}_{1z} 
			&=  t_{2z} - t_{1z} + m_2 \tilde{d}_{2z} - m_1 \tilde{d}_{1z} + \tilde{\Xi}_z\\
			&=  t_0 \sin \theta_0 + \bm{C} \tilde{\bm{d}}
	\end{aligned}
\end{equation}
where 
$
\bm{C} = \left[ \frac{-2 m_1}{\tan \theta_2 - \tan \theta_1}, 0, \frac{2 m_1 \tan \theta_1}{\tan \theta_2  -\tan \theta_1}, \frac{-2 m_2}{\tan \theta_2 - \tan \theta_1}, 0, \right.\\ 
\left. \frac{2 m_2 \tan \theta_2}{\tan \theta_2 - \tan \theta_1}\right]
$ and $\tilde{\bm{d}}$ is notated in \eqref{zeta}.

The vertical position and vertical velocity tracking error of quadrotor 2 are denoted by
the notations $(\tilde{p}_{2z}, \dot{\tilde{p}}_{2z} )$.
	 Combining \eqref{control model}, \eqref{vertical controller}, and \eqref{force consensus estimate},
	 the tracking error dynamics for quadrotor 2 in the vertical direction is 
	 \begin{equation}
	 	\label{p2z}
	 	\begin{aligned}
			 		\ddot{\tilde{p}}_{2z} 
		 		&= -k_4 \dot{\tilde{p}}_{2z} - \tilde{d}_{2z} + k_f t_0 \sin \theta_0 + k_f \bm{C} \tilde{\bm{d}}. 
	 	\end{aligned}
	 \end{equation}
\begin{lemma}\label{lemma1}
$t_0 \sin \theta_0$ and  $\tilde{p}_{2z}$ are negatively correlated and satisfy the following equation
\begin{equation}
	t_0 \sin \theta_0 = -\sigma \left(\tilde{p}_{2z}\right)
\end{equation}
where $\sigma (x)$ is a strictly increasing function with $\sigma\left(0\right) = 0$. Moreover, the slope of $\sigma (x)$ satisfies 
\begin{equation}
0 < \underline{\sigma} < \frac{d \sigma (x)}{d x} < \overline{\sigma}
\end{equation}
where $\underline{\sigma} \in \mathbb{R}^{+}$ and $\overline{\sigma} \in \mathbb{R}^{+}$ are the constant lower and upper bounds.
\end{lemma}
\begin{proof}
	Detailed proof can be found in the supplementary material.
\end{proof}
Based on Lemma \ref{lemma1}, equation \eqref{p2z} can be turned into
\begin{equation}
	\begin{aligned}
		\ddot{\tilde{p}}_{2z} = - k_4 \dot{\tilde{p}}_{2z} - k_f t_0 \sigma\left(\tilde{p}_{2z}\right) + \bm{D} \tilde{\bm{d}}
	\end{aligned}
\end{equation}
where $\bm{D}  = \left[ \frac{-2 k_f m_1}{\tan \theta_2 - \tan \theta_1}, 0, \frac{2 k_f m_1 \tan \theta_1}{\tan \theta_2  -\tan \theta_1}, \frac{-2 k_f  m_2}{\tan \theta_2 - \tan \theta_1}, 0,\right.\\ 
\left.  \frac{\left(2 k_f  m_2 - 1\right) \tan \theta_2 + \tan \theta_1}{\tan \theta_2 - \tan \theta_1}\right]$.
According to \eqref{theta assumption}, there exists unknown positive constant $\bar{D}$ satisfying
\begin{equation}
	\|\bm{D}\| \le \bar{D}.
\end{equation}
Here the boundedness of $\|\bm{D}\|$ can be derived since the absolute value of the denominator $|\tan \theta_2 - \tan \theta_1|$ is lower bounded by the positive constant $|\tan \overline{\theta}_2 - \tan \underline{\theta}_1|$.

\begin{assumption}
	There exist positive constants $\varepsilon_1$, $\gamma$, $k_{4}$, and $k_f$ satisfying
	\begin{equation}\label{controlgainassumptipon}
		\begin{aligned}
			k_4 &> \varepsilon_1 + \frac{\overline{t}_0}{4 \gamma}\\
			k_f &< \frac{\varepsilon_1 \underline{\sigma}}{\gamma \overline{\sigma}^2}
		\end{aligned}
	\end{equation} 
	where details about  $\overline{t}_0$, $\overline{\sigma}$ and $\underline{\sigma}$ can be found in the supplementary material.
\end{assumption}
\begin{theorem}
	Consider two quadrotors carrying a suspended payload under the quasi-static condition, modeled as \eqref{control model},
	 with the control laws \eqref{force controller 2},
	the disturbance estimate updated as in \eqref{disturbance estimate},
	and the inconsistency between the thrust uncertainties estimated as in \eqref{inconsistency estimate}.
	Suppose that Assumptions 1-3 hold. 
	Then the signals of the overall closed-loop system, $(\tilde{\bm{p}}_{1xy},
		\dot{\tilde{\bm{p}}}_{1xy}, \tilde{\bm{p}}_{12xy}, \dot{\tilde{\bm{p}}}_{12xy})$,  ($\tilde{p}_{1z}, \dot{\tilde{p}}_{1z}, \tilde{p}_{2z}, \dot{\tilde{p}}_{2z}$), and ($\tilde{\bm{d}}_1, \tilde{\bm{d}}_2$), are uniformly ultimately bounded.
\end{theorem}
\begin{proof}
	The boundedness of $\tilde{\bm{d}}_i$ is already established in \eqref{disturbance bounded}.
	The horizontal controller \eqref{horizontal controller} and the vertical controller \eqref{tracking controller1} \eqref{vertical controller} are designed independently.
	Since the horizontal controller \eqref{horizontal controller} is the same as that in the position-coordination control scheme, the boundedness of the states in the horizontal plane $(\tilde{\bm{p}}_{1xy},
	\dot{\tilde{\bm{p}}}_{1xy}, \tilde{\bm{p}}_{12xy}, \dot{\tilde{\bm{p}}}_{12xy})$ can be  proved. 
	In addition, the boundedness of $(\tilde{p}_{1z}, \dot{\tilde{p}}_{1z})$ is already established after \eqref{exponential stability}.
	
	We focus on the stability of quadrotor 2 in the vertical direction.
Define the Lyapunov function as
\begin{equation}
	\label{lyapunov}
	V = \frac{1}{2} 
	\begin{bmatrix}
		\tilde{p}_{2z} & \dot{\tilde{p}}_{2z}
	\end{bmatrix}
\begin{bmatrix}
	k_4 \varepsilon_1 & \varepsilon_1\\
	\varepsilon_1 & 1
\end{bmatrix}
\begin{bmatrix}
	\tilde{p}_{2z}\\
	\dot{\tilde{p}}_{2z}
\end{bmatrix}
\end{equation} 
where $\varepsilon_1 \in \mathbb{R}^{+}$ is referred in Assumption 2.

Differentiate \eqref{lyapunov} with respect to time $t$
\begin{equation}
\begin{aligned}
	\dot{V} =& 
	\begin{bmatrix}
		\tilde{p}_{2z} & \dot{\tilde{p}}_{2z}
	\end{bmatrix}
	\begin{bmatrix}
		k_4 \varepsilon_1 & \varepsilon_1\\
		\varepsilon_1 & 1
	\end{bmatrix}
	\begin{bmatrix}
		\dot{\tilde{p}}_{2z}\\
		\ddot{\tilde{p}}_{2z}
	\end{bmatrix}\\
	=& k_4\varepsilon_1  \tilde{p}_{2z} \dot{\tilde{p}}_{2z}+ \varepsilon_1 \tilde{p}_{2z} \ddot{\tilde{p}}_{2z}
	+ \varepsilon_1 \dot{\tilde{p}}_{2z}^2 + \dot{\tilde{p}}_{2z} \ddot{\tilde{p}}_{2z} \\
	=& - k_f t_0 \varepsilon_1 \tilde{p}_{2z} \sigma\left(\tilde{p}_{2z}\right)
	-\left(k_4 - \varepsilon_1\right) \dot{\tilde{p}}_{2z}^2\\
	&- k_f t_0 \dot{\tilde{p}}_{2z} \sigma \left(\tilde{p}_{2z}\right)+ \dot{\tilde{p}}_{2z} \bm{D} \tilde{\bm{d}} + \varepsilon_1 \tilde{p}_{2z} \bm{D} \tilde{\bm{d}}\\
	\le & -k_f  t_0  \varepsilon_1  \underline{\sigma} \tilde{p}_{2z}^2
	- \left(k_4 - \varepsilon_1\right) \dot{\tilde{p}}_{2z}^2 \\
	&- k_f t_0 \dot{\tilde{p}}_{2z} \sigma \left(\tilde{p}_{2z}\right)
	+ \dot{\tilde{p}}_{2z} \bm{D} \tilde{\bm{d}} + \varepsilon_1 \tilde{p}_{2z} \bm{D} \tilde{\bm{d}}.
\end{aligned}
\end{equation}
Using Young's inequality yields
\begin{equation}
	\begin{aligned}
		- k_f t_0 \dot{\tilde{p}}_{2z} \sigma \left(\tilde{p}_{2z}\right) \le& \frac{{t}_0 }{4 \gamma} \dot{\tilde{p}}_{2z}^2 + \gamma k_f^2 t_0  \sigma^2 \left(\tilde{p}_{2z}\right) \\
		\le&     \frac{{t}_0 }{4 \gamma} \dot{\tilde{p}}_{2z}^2 + \gamma k_f^2 t_0  \overline{\sigma}^2 \tilde{p}_{2z}^2 
	\end{aligned}
\end{equation}
where $\gamma \in \mathbb{R}^{+}$ is the tuning parameter.

Therefore, $\dot{V}$ satisfies 
\begin{equation}
	\begin{aligned}
		\dot{V} \le& - \left(k_f t_0 \varepsilon_1 \underline{\sigma} - \gamma k_f^2 t_0  \overline{\sigma}^2\right) \tilde{p}_{2z}^2 - \left( k_4 - \varepsilon_1  
		- \frac{t_0}{4 \gamma}\right) \dot{\tilde{p}}_{2z}^2\\
		&+ \|\dot{\tilde{p}}_{2z}\| \| \bm{D}\| \| \tilde{\bm{d}}\| + \varepsilon_1 \|\tilde{p}_{2z}\|  \|\bm{D}\|  \|\tilde{\bm{d}}\|\\
		\le&  - k_f t_0 \left( \varepsilon_1 \underline{\sigma} - \gamma k_f   \overline{\sigma}^2\right) \tilde{p}_{2z}^2 - \left( k_4 - \varepsilon_1  
		- \frac{t_0}{4 \gamma}\right) \dot{\tilde{p}}_{2z}^2\\
		&+ \sqrt{2}  \bar{D} \bar{\tilde{d}} \|\dot{\tilde{p}}_{2z}\|   + \sqrt{2}\varepsilon_1  \bar{D} \bar{\tilde{d}} \|\tilde{p}_{2z}\| 
	\end{aligned}
\end{equation}
where $\| \tilde{\bm{d}} \| \le \sqrt{2} \bar{\tilde{d}}$ is used.
 According to Lyapunov bounded theory \cite{khalil2002nonlinear}, 
the error variables $\left(\tilde{p}_{2z}, \dot{\tilde{p}}_{2z}\right)$ are uniformly ultimately bounded if \eqref{controlgainassumptipon} is satisfied.
\end{proof}
\section{Numerical Simulations} \label{section 5}
Numerical simulation is carried out first to validate the proposed force-coordination control strategy. Specially,  cables of different unknown lengths and different thrust uncertainties are configured for the two quadrotors to demonstrate the effects of disturbance separation and force-consensus.
In simulation, the cables are assumed to be massless links and the cable forces are modeled following \cite{Taeyoung2018}.
System and control parameters are listed in Table \ref{tab}.
The thrust uncertainties for the two quadrotors are set as $\Delta f_1 = - 0.2 f_{1c}$ and $\Delta f_2 = -0.4 f_{2c}$, respectively. 
The simulation time is set as 30 seconds.
Quadrotor 2 is free to adjust its height to achieve force-consensus.  According to the equilibrium analysis in Section \ref{Equilibrium Analysis}, the pipe will eventually become parallel to the ground. 

\begin{table}[h]
	\centering
	\caption{Parameters used in the simulation}
	\resizebox{!}{!}{
\begin{tabular}{c c c} 
			\hline 
			\hline
			\multicolumn{3}{c}{System Parameters}\\
			\hline
			\multirow{3}*{$m_i$} & i = 0 & $0.44,  \mathrm{kg}$\\
			 & i = 1 & $0.87, \mathrm{kg}$\\
			& i = 2 & $0.88, \mathrm{kg}$\\
			\hline
			\multirow{3}*{$\bm{J}_i $} & i = 0 & $\mathrm{diag}\{ 0.0035, 0.15, 0.15\}, \mathrm{kg} \cdot \mathrm{m}^2$\\
			& i = 1 & $\mathrm{diag}\{ 0.003, 0.003, 0.004\}, \mathrm{kg} \cdot \mathrm{m}^2$\\
			& i = 2 & $\mathrm{diag}\{0.003, 0.003, 0.004\}, \mathrm{kg} \cdot \mathrm{m}^2$\\
			\hline
			\multirow{3}*{$l_i $} & i = 0 & $1.0, \mathrm{m}$\\
			& i = 1 & $0.8, \mathrm{m}$\\
			& i = 2 & $0.4, \mathrm{m}$\\
			\hline
			\multicolumn{3}{c}{Control Parameters}\\
			\hline
			\multicolumn{2}{c}{$k_1, k_2, k_3, k_4$} & $4, 4, 5, 8$\\
			\multicolumn{2}{c}{$k_f$} & $0.5$\\
			\multicolumn{2}{c}{$\iota$} & $5$\\
			\hline
			\multicolumn{3}{c}{Initial Conditions}\\
			\hline
			\multicolumn{2}{c}{$\bm{p}_{0}(0)$} & $[0, 0, 0]^\top, \mathrm{m}$\\
			\multicolumn{2}{c}{$\bm{p}_{1}(0)$} & $[1.4, 0.12, -0.68]^\top, \mathrm{m}$\\
			\multicolumn{2}{c}{$\bm{p}_2 (0)$} & $[-1.14, 0, -0.38]^\top, \mathrm{m}$\\
			\multicolumn{2}{c}{$\dot{\bm{p}}_i (0)$, $i = 0, 1, 2$} & $[0, 0, 0]^\top, \mathrm{m/s}$\\
			\multicolumn{2}{c}{$\bm{R}_i (0)$, $i = 0, 1, 2$} & $[\bm{e}_1, \bm{e}_2, \bm{e}_3]$\\
			\multicolumn{2}{c}{$\bm{\Omega}_i(0)$, $i = 0, 1, 2$} & $[0, 0, 0]^\top, \mathrm{rad/s}$\\
			\multicolumn{2}{c}{$\bm{z}_i(0)$, $i = 1, 2$} & $[0, 0, 0]^\top $\\
			\hline
			\multicolumn{3}{c}{Desired Trajectories}\\
			\hline
			\multicolumn{2}{c}{$\bm{p}_{1d}$} & $[1, 0, -1]^\top, \mathrm{m}$\\
			\multicolumn{2}{c}{$\bm{p}_{12dxy}$} & $[2.5, 0]^\top, \mathrm{m}$\\
			\hline
			\hline
\end{tabular}}
\label{tab}
\end{table}

To test the stability of the equilibrium,  the aerial transportation system works in the position-coordination control mode in the first 10 seconds, so the pipe leans towards approximately $-10^\circ$ in the end as shown in Fig. \ref{fig_14}. Then  the controllers switch to the force-coordination mode.
Since quadrotor 2 has the shorter cable, it burdens more pipe weight than quadrotor 1 under the position-consensus condition. Once in the force-coordination mode, quadrotor 2 comes down slowly, making the pipe finally parallel to the ground as shown in Fig. \ref{fig_12}, which implies the stability of the equilibrium.

\begin{figure}[!hbt]
	\centering
	\includegraphics[width=3.5in]{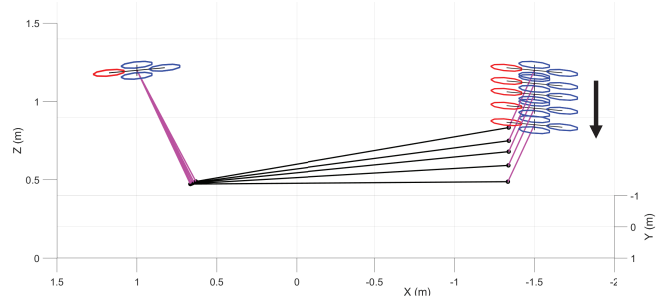}
	\caption{Movement snapshots of force-consensus simulation.}
	\label{fig_12}
\end{figure}

\begin{figure}[!hbt]
	\centering
	\includegraphics[width=3.5in]{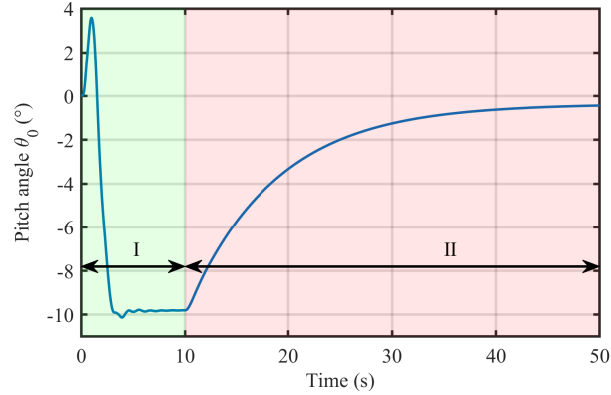}
	\caption{Pitch angle $\theta_0$ of the pipe. Stage \Rmnum{1} stands for the position-coordination period and  Stage \Rmnum{2} stands for the force-coordination period.}
	\label{fig_14}
\end{figure}

The estimation errors of vertical cable forces are presented in Fig. \ref{fig_13}.
The bounded convergence property is exhibited. 
Subject to the disturbance estimation error $\tilde{\bm{d}}_i$, both the cable force estimation error in Fig. \ref{fig_13} and the pitch angle in  Fig. \ref{fig_14} can only converge to a small neighborhood of zero.

\begin{figure}[!hbt]
	\centering
	\includegraphics[width=3.5in]{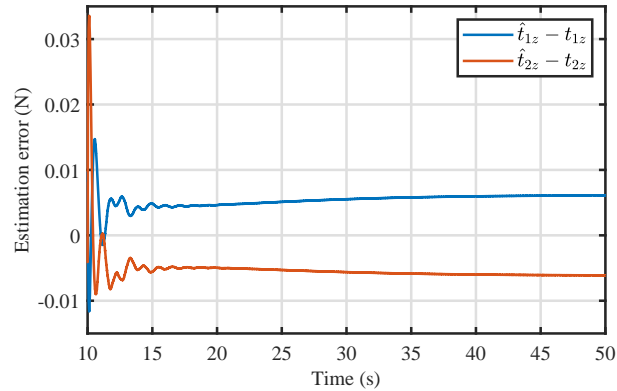}
	\caption{Estimation errors of the vertical cable forces.}
	\label{fig_13}
\end{figure}

\begin{figure*}[htbp]
	\centering
	\includegraphics[width=6in]{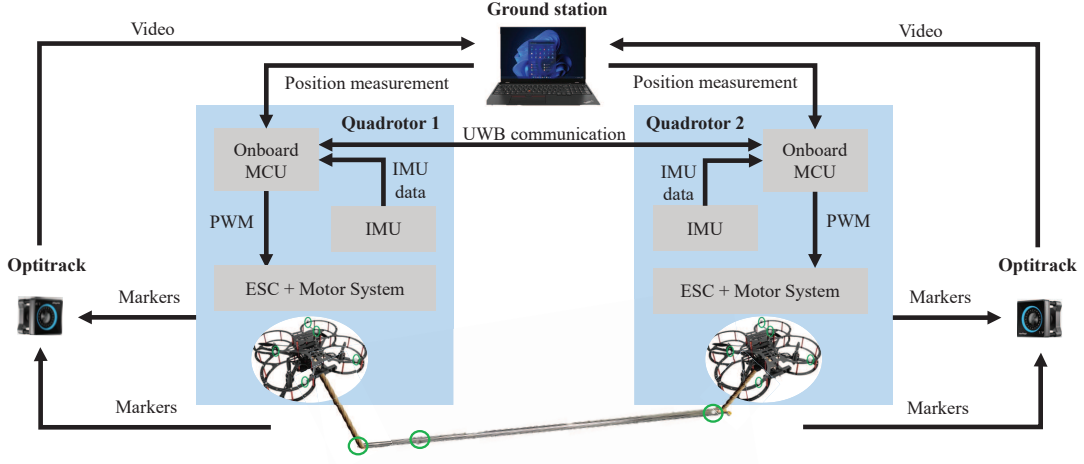}
	\caption{Experimental platform structure.}
	\label{fig_15}
\end{figure*}

\section{Experiment Validation}\label{section 6}
To demonstrate the effectiveness of the proposed force-coordination algorithm in practical implementation, real-world flight tests are performed in an indoor flight test environment. The test facility consists of Optitrack motion capture system, ground station, and quadrotor platforms, as shown in Fig. \ref{fig_15}, which has been used to support different research projects (see e.g.,  \cite{jindoujia, cailiu, zhang2022safety} for details).
The first test is intended to illustrate the significant inconsistency of the thrust uncertainties between the two quadrotors, 
followed by the main test for the aerial transportation system.

System and control parameters are listed in Table \ref{tab2}.
Due to the measurement noises,
the disturbance observer  gain $\iota$ cannot be selected too large.
Otherwise, it will result in large chattering in the disturbance estimate,
which may lead to sudden and large swings of the quadrotor.
The force-consensus gain $k_f$ is selected small as well.

\begin{table}[h]
	\centering
	\caption{Parameters used in the experiments}
	\resizebox{!}{!}{
		\begin{tabular}{c c c} 
			\hline 
			\hline
			\multicolumn{3}{c}{System Parameters}\\
			\hline
			\multirow{3}*{$m_i$} & i = 0 & $0.745, \mathrm{kg}$\\
			& i = 1 & $0.831, \mathrm{kg}$\\
			& i = 2 & $0.832, \mathrm{kg}$\\
			\hline
			\multirow{2}*{$\bm{J}_i $} & i = 1 & $\mathrm{diag}\{ 0.003, 0.003, 0.004\}, \mathrm{kg} \cdot \mathrm{m}^2$\\
			& i = 2 & $\mathrm{diag}\{0.003, 0.003, 0.004\}, \mathrm{kg} \cdot \mathrm{m}^2$\\
			\hline
			\multirow{3}*{$l_i $} & i = 0 & $1.0, \mathrm{m}$\\
			& i = 1 & $0.7, \mathrm{m}$\\
			& i = 2 & $0.3, \mathrm{m}$\\
			\hline
			\multicolumn{3}{c}{Control Parameters}\\
			\hline
			\multicolumn{2}{c}{$k_1, k_2, k_3, k_4$} & $4, 4, 5, 8$\\
			\multicolumn{2}{c}{$k_f$} & $0.1$\\
			\multicolumn{2}{c}{$\iota$} & $1.2$\\
			\hline
			\multicolumn{3}{c}{Desired Trajectories}\\
			\hline
			\multicolumn{2}{c}{$\bm{p}_{1d}$} & $[1.5, 0, -1.2]^\top, \mathrm{m}$\\
			\multicolumn{2}{c}{$\bm{p}_{12dxy}$} & $[2.5, 0]^\top, \mathrm{m}$\\
			\hline
			\hline
	\end{tabular}}
	\label{tab2}
\end{table}

\subsection{Hovering experiment}
Two quadrotors carrying objects of similar weight are required to hover at the same height.
Specifically, 
quadrotor 1 carries a payload of $0.248 \mathrm{kg}$ 
and  quadrotor 2 carries a payload  of $0.252 \mathrm{kg}$.
We adopt the disturbance observer based position controller for both quadrotors.
The desired height for the two quadrotors are chosen as $1.2 \mathrm{m}$.

According to the definition \eqref{control input def} and the model \eqref{control model}, under the hovering condition of $\ddot{\bm{p}}_i = \bm{0}_{3\times 1}$, the actual values of the lumped disturbances in the vertical direction for the two quadrotors  can be approximately computed as $m_i d_{iz} =  - m_i g + f_{ic}$. 
It can be observed from  Fig. \ref{fig_16} that the disturbance estimate $m_i \hat{d}_{iz}$ converges to a small neighborhood of the actual value $m_i d_{iz}$,
 implying the effectiveness of the disturbance observer.
Although the whole weight of the two quadrotor-payload units are almost same, 
$m_2 d_{2z}$ is nearly $1 \mathrm{N}$ larger than $m_1 d_{1z}$.
It can inferred that the inconsistency of the thrust uncertainties between the two quadrotors reaches about $10\%$ of the total gravity and $40\%$ of the payload gravity  (whole weight of quadrotor 1 with payload is 1.079kg and payload is 0.248kg), which implies that even drones of the same type may have severe inconsistency of uncertainties.
According to the experimental experience, this inconsistency may be amplified by the differences between battery levels.
In this case, it is unfeasible to apply the force control methods proposed in  \cite{thapa2020cooperative} or \cite{Tognon}.
Therefore, we have to separate the cable force from the lumped disturbance for force-coordination control.

\begin{figure}[!hbt]
	\centering
	\includegraphics[width=3.5in]{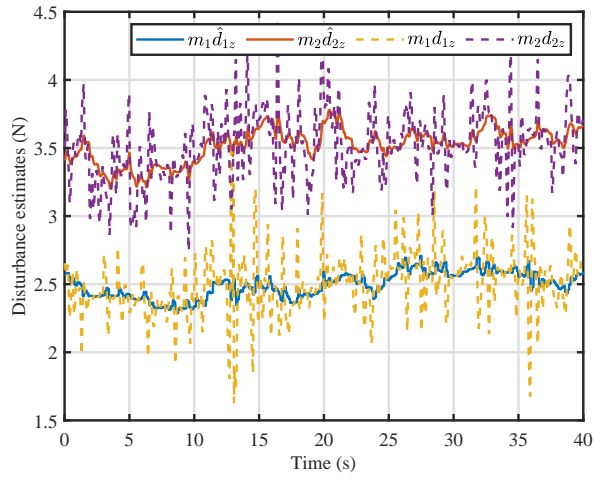}
	\caption{Disturbance estimates in the vertical direction.}
	\label{fig_16}
\end{figure}

\subsection{Force-consensus-based experiment}
\begin{figure}[!hbt]
	\centering
	\includegraphics[width=3.5in]{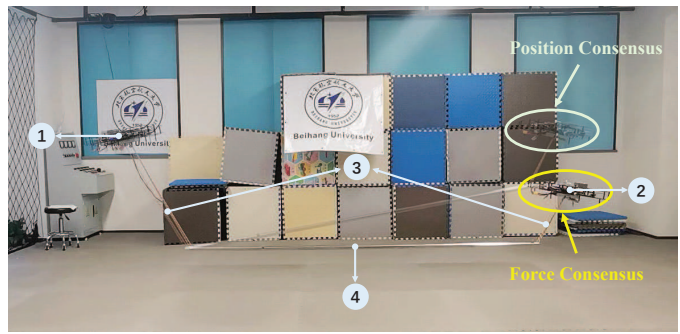}
	\caption{1) quadrotor 1 acting as the leader; 2) quadrotor 2 acting as the follower; 3) cables of different lengths; 4) the steel pipe}
	\label{fig_1}
\end{figure}	

The experiment configuration is shown in Fig. \ref{fig_15}.
Quadrotor 1 is required to hover at the desired position $\bm{p}_{1d}$,
while quadrotor 2 is commanded to maintain the desired horizontal relative position $\bm{p}_{12dxy}$
and achieve force-consensus in the vertical direction.
The choice of $\bm{p}_{12d}$ is based on the requirement that
$\|\bm{p}_{12dxy}\|$ has to be larger than the length of the pipe as explained in \eqref{formation setting} and less than the sum of the lengths of the cables and the length of the pipe, i.e., $2 l_0 < \| \bm{p}_{12dxy}\| < 2l_0 + l_1 + l_2$. 

Same as  in the simulation studies,
the experiment also starts from position-coordination control mode and then changes to force-coordination control mode at $10$s.
The position tracking errors of quadrotor 1 and quadrotor 2 are shown respectively in Fig. \ref{fig_17} and Fig. \ref{fig_18}.
It can be observed that
the position tracking errors fluctuate in the  range of $0.02 \mathrm{m}$ roughly.
Therefore, combining with the estimation results shown in Fig. \ref{fig_16}, the robustness of the DO-based controller to the lumped disturbance is well demonstrated.
\begin{figure}[!hbt]
	\centering
	\includegraphics[width=3.5in]{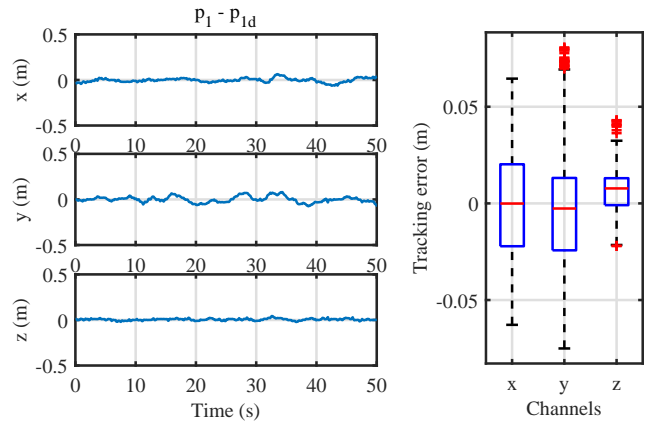}
	\caption{Tracking performance of quadrotor 1.}
	\label{fig_17}
\end{figure}

\begin{figure}[!hbt]
	\centering
	\includegraphics[width=3.5in]{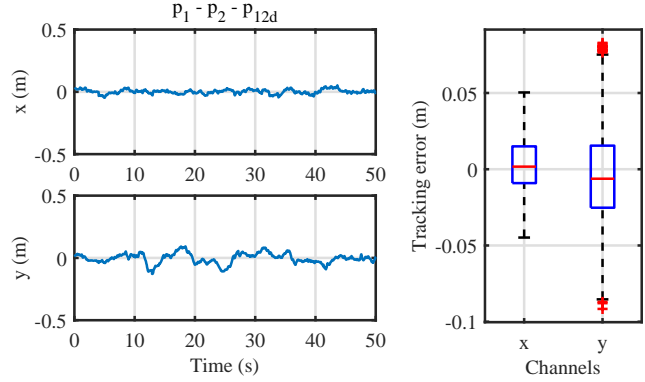}
	\caption{Formation-keeping performance of quadrotor 2.}
	\label{fig_18}
\end{figure}

The estimates of thrust uncertainties $\Delta f_1$ and $\Delta f_2$ are shown in Fig. \ref{fig_19}.
	It can be observed that the thrust uncertainty estimate $\Delta \hat{f}_2$ for quadrotor 2 is larger than $\Delta \hat{f}_1$ for quadrotor 1,
	which matches with the comparison result in the hovering experiment.
The evolution of the pitch angle  $\theta_0$ of the pipe is shown in Fig. \ref{fig_20}.
During the first 10 seconds,
the pipe keeps static at the inclined posture of $\theta_0 \approx -14^\circ$.
In the subsequent force-coordination control mode, the pipe  approaches the equilibrium of the horizontal  posture gradually, as quadrotor 2 goes down slowly over 20 seconds.
Actually, even though the dominant thrust uncertainty has been removed from the lumped force disturbance estimate, there are still some extra trivial uncertainties existing in the residual disturbance estimate, which is regarded as the vertical cable force estimate $\hat{t}_{iz}$. 
As a result, the final pitch angle $\theta_0$ of the pipe can only stay in the interval of $1^\circ$ and $3^\circ$. 
Here the extra model uncertainties refer to the mass of the cable, the deviation of the CoM of the quadrotor, the acceleration of the pipe, and the turbulence in the test area,
which are usually small compared to the thrust uncertainty, so that they are neglected in this study.

\begin{figure}[!hbt]
	\centering
	\includegraphics[width=3.5in]{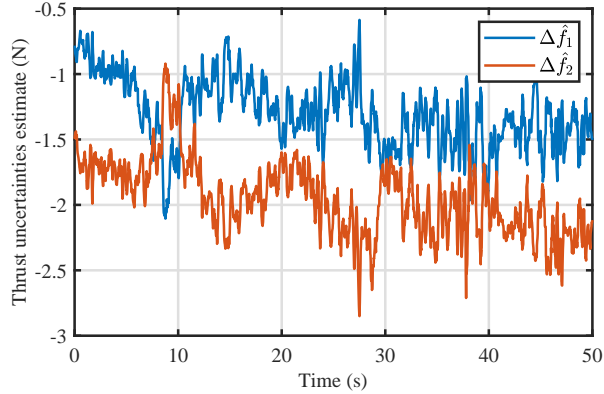}
	\caption{Estimates of the thrust uncertainties.}
	\label{fig_19}
\end{figure}

\begin{figure}[!hbt]
	\centering
	\includegraphics[width=3.5in]{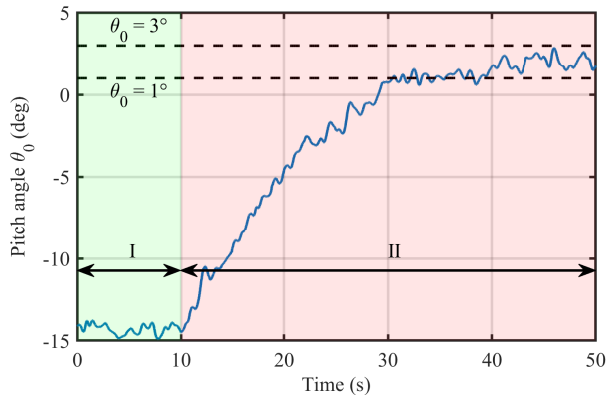}
	\caption{Pitch angle $\theta_0$ of the pipe.
		 Stage \Rmnum{1} stands for the position-coordination period and  Stage \Rmnum{2} stands for the force-coordination period.}
	\label{fig_20}
\end{figure}

\section{Conclusions}\label{section 7}
In this article, a force-coordination control scheme with disturbance separation and estimation is proposed, as demonstrated for a collaborative transportation system. Compared to position-coordination control, force-coordination control can provide more complex manipulation of the payload than simply moving along the predefined trajectory. Under the quasi-static condition, the force-consensus objective can ensure that vehicles share the same weight of the payload, which can extend the endurance of the entire transportation mission. By exploiting the intrinsic force balance conditions of the cooperative quadrotors, thrust uncertainty can be separately estimated from the lumped force disturbance. Therefore, a more accurate cable force estimate can be obtained by removing thrust uncertainty. This overcomes the problem that the existing disturbance estimation methods cannot distinguish the different disturbances in the same channel. Simulation results verify the effectiveness of the proposed method and experiments further demonstrate that the proposed method can achieve good performance in practical implementation of payload transportation using heterogeneous quadrotors. Future research directions include payload attitude control through force-coordination and separating other undesirable force disturbances, such as wind.


\pagebreak

\begin{center}
	\textbf{\large Supplemental Material: 
		Sector Bounds for Vertical Cable Force Error in Cable-Suspended Load Transportation System}
\end{center}

	\begin{figure}[h]
	\centering
	\includegraphics[width=5in]{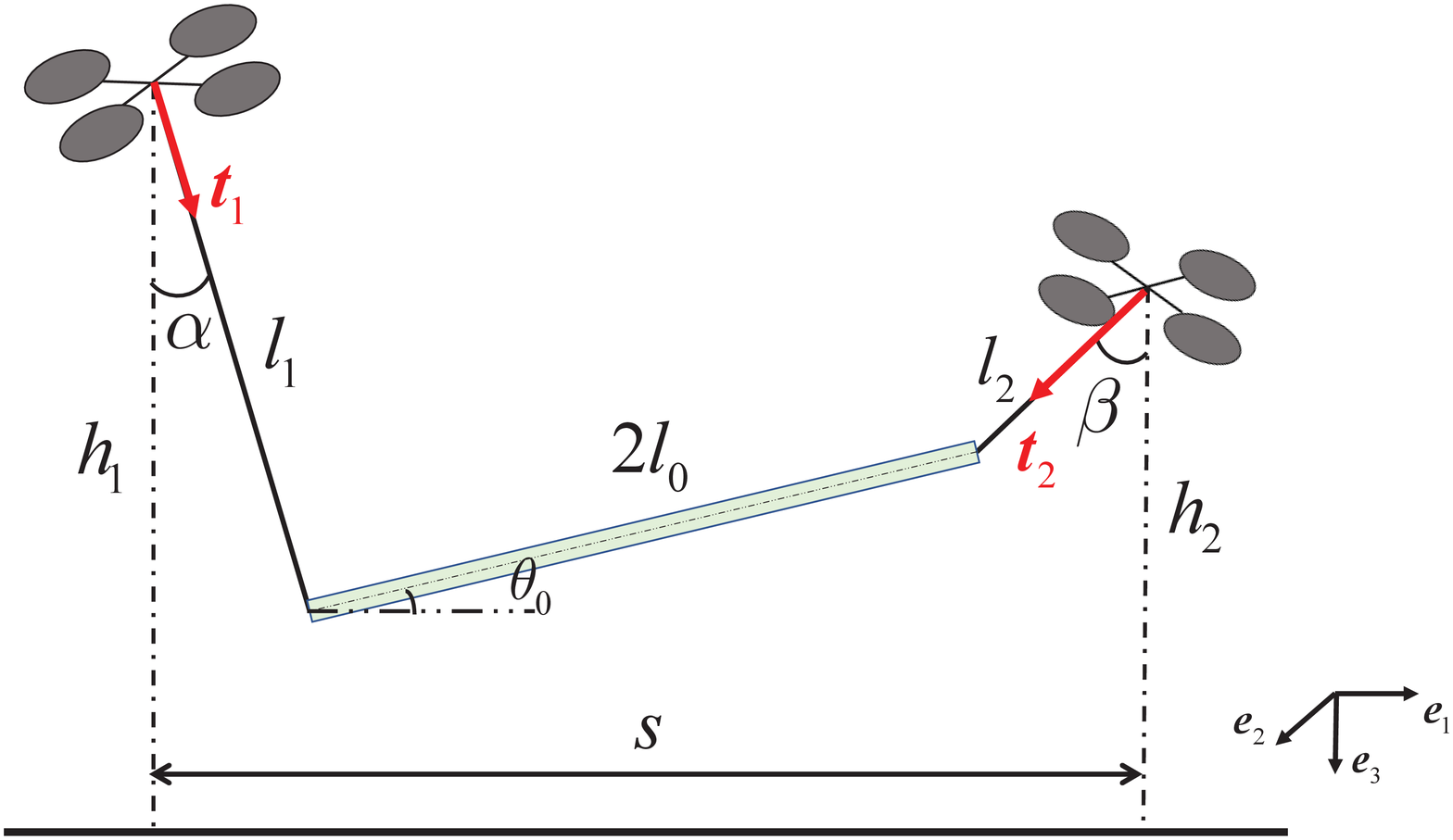}
	\caption{Regulation of the height of quadrotor 2.}
	\label{fig_22}
\end{figure}

\subsection{Proof for Lemma 1}

\begin{proof}
	First, considering the scene where the internal force $t_0 > 0$, the cable angles $\alpha, \beta$ and the load angle $\theta_0$ in Fig. \ref{fig_22} satisfy
	\begin{equation}
		\begin{aligned}
			0 < \alpha < \frac{\pi}{2},
			\quad 0 < \beta < \frac{\pi}{2},
			\quad -\frac{\pi}{2} < \theta_0 < \frac{\pi}{2}.
		\end{aligned}
	\end{equation}
	As the payload stays in the XZ plane of  NED frame under the quasi-static condition, 
	the cable forces can be computed as \cite{Tognon}
	\begin{equation}\label{balance force}
		\begin{aligned}
			\bm{t}_1 &= 
			\left[	\frac{t_0 \cos \theta_0 }{2}, 0, \frac{-t_0 \sin \theta_0 + m_0 g}{2} 
			\right]^\top, \quad
			\bm{t}_2 =
			\left[	-\frac{ t_0 \cos \theta_0}{2}, 0, \frac{t_0 \sin \theta_0 + m_0 g}{2}
			\right]^\top
		\end{aligned}
	\end{equation}
	where $\theta_0$ is the pitch angle of the pipe notated in Fig. \ref{fig_22}.
	
	From \eqref{balance force}, trigonometric functions of $\alpha$ and $\beta$  are computed as
	\begin{equation}\label{substitutions}
		\begin{aligned}
			\cos \alpha = &\frac{-t_0 \sin \theta_0 + m_0 g}{\sqrt{\left(t_0 \cos \theta_0\right)^2 + \left(-t_0 \sin \theta_0 + m_0 g\right)^2}} &&= \frac{k - \sin \theta_0}{g_1}\\
			\cos \beta = &\frac{t_0 \sin \theta_0 + m_0 g}{\sqrt{\left(t_0 \cos \theta_0\right)^2 + \left(t_0 \sin \theta_0 + m_0 g\right)^2}} &&= \frac{k + \sin \theta_0}{g_2}\\
			\sin \alpha = &\frac{t_0 \cos \theta_0}{\sqrt{\left(t_0 \cos \theta_0\right)^2 + \left(-t_0 \sin \theta_0 + m_0 g\right)^2}} &&= \frac{\cos \theta_0}{g_1}\\
			\sin \beta = &\frac{t_0 \cos \theta_0}{\sqrt{\left(t_0 \cos \theta_0\right)^2 + \left(t_0 \sin \theta_0 + m_0 g\right)^2}} &&= \frac{\cos \theta_0}{g_2}
		\end{aligned}
	\end{equation}
	where $k = \frac{m_0 g}{t_0}$, $g_1 = \sqrt{k^2 - 2 k \sin \theta_0 + 1}$, and 
	$g_2 = \sqrt{k^2 + 2 k \sin \theta_0 + 1}$ are used for substitutions.
	Here the derivatives of $g_1$ and $g_2$ with respect to $\theta_0$ are computed as
	\begin{equation}
		\begin{aligned}
			\frac{d g_1}{d \theta_0} &= \frac{1}{g_1} \left[(k - \sin \theta_0) \frac{d k}{d \theta_0} - k \cos \theta_0\right], \quad
			\frac{d g_2}{d \theta_0} = \frac{1}{g_2} \left[(k + \sin \theta_0) \frac{d k}{d \theta_0} + k \cos \theta_0\right].
		\end{aligned}
	\end{equation}
	
	The internal force $t_0$ satisfies the following constraint equation
	\begin{equation}
		\label{constraint}
		\begin{aligned}
			l_1 \sin \alpha + l_2 \sin \beta + 2 l_0 \cos \theta_0 = s.
		\end{aligned}
	\end{equation}
	Using the substitutions in \eqref{substitutions} yields
	\begin{equation}\label{constraint2}
		\begin{aligned}
			l_1 \frac{\cos \theta_0}{g_1} + l_2 \frac{\cos \theta_0}{g_2} + 2 l_0 \cos \theta_0 = s.
		\end{aligned}
	\end{equation}
	The constraint equation \eqref{constraint2} is a high-order equation with respect to $k$,
	so
	it is quite hard to express the variable $k$ as an analytic function of $\theta_0$.
	Differentiating \eqref{constraint2} with respect to $\theta_0$  yields
	\begin{equation}\label{deri k}
		\begin{aligned}
			\frac{l_1}{g_1^3} [(k - \sin \theta_0) \frac{d k}{d \theta_0} - k \cos \theta_0] 
			+ \frac{l_2}{g_2^3} [(k + \sin \theta_0) \frac{d k}{d \theta_0} + k \cos \theta_0]
			= -\frac{s \cdot \sin \theta_0}{\cos^2 \theta_0}.
		\end{aligned}
	\end{equation}
	Combining \eqref{constraint2}, the derivative of $k$ with respect to $\theta_0$ is computed as
	\begin{equation}
		\begin{aligned}
			\frac{d k}{d \theta_0} = - \frac{\frac{l_1}{g_1^3}(k - \sin \theta_0) (k \sin \theta_0 - 1)  + \frac{l_2}{g_2^3} (k + \sin \theta_0)(k \sin \theta_0 + 1) + 2 l_0 \sin \theta_0}{\left[\frac{l_1 (k - \sin \theta_0)}{g_1^3} + \frac{l_2 (k + \sin \theta_0)}{g_2^3}\right] \cos \theta_0}
		\end{aligned}
	\end{equation}
	
	Next, the height of quadrotor 2 is calculated as
	\begin{equation}
		\begin{aligned}
			h_2 &= l_2 \cos \beta  + 2 l_0 \sin \theta_0 - l_1 \cos \alpha + h_1\\
			&= l_2 \frac{k+\sin \theta_0}{g_2} + 2 l_0 \sin \theta_0 - l_1 \frac{k-\sin \theta_0}{g_1} + h_1
		\end{aligned}
	\end{equation}
	where  $h_1 = -p_{1z}$ is assumed to be static at the steady state.
	The lengths of cables $l_1$, $l_2$ and the length of pipe $2 l_0$ are also fixed.
	Then $h_2$ can be seen as a continuous function of  $\theta_0$.  
	Differentiating $h_2$ with respect to $\theta_0$ yields
	\begin{equation}
		\begin{aligned}
			\frac{d h_2}{d \theta_0} &= l_2 \frac{(\frac{d k}{d \theta_0} + \cos \theta_0) g_2 - (k + \sin \theta_0) \frac{d g_2}{d \theta_0}}{g_2^2}
			- l_1 \frac{(\frac{d k}{d \theta_0} - \cos \theta_0) g_1 - (k - \sin \theta_0) \frac{d g_1}{d \theta_0}}{g_1^2} + 2 l_0 \cos \theta_0\\
			&= (\frac{l_2}{g_2^3} - \frac{l_1}{g_1^3}) \frac{d k}{d \theta_0} \cos^2 \theta_0 + \frac{l_2 \cos \theta_0 (k \sin \theta_0 + 1)}{g_2^3} + \frac{l_1 \cos \theta_0 (1 - k \sin \theta_0)}{g_1^3} + 2 l_0 \cos \theta_0\\
			&= k \cos \theta_0 \frac{\frac{4 l_1 l_2}{g_1^3 g_2^3} + 2 l_0 (\frac{l_2}{g_2^3}+\frac{l_1}{g_1^3})}{\frac{l_1(k - \sin \theta_0)}{g_1^3} + \frac{l_2 (k + \sin \theta_0)}{g_2^3}}.
		\end{aligned}
	\end{equation}
	Since $p_{1dz}$ and $p_{12 d}$ are fixed, 
	$p_{2dz}$ is also fixed according to the equilibrium analysis in the article, i.e., $\dot{p}_{2dz} = 0$.
	Therefore, the following equation is obtained 
	\begin{equation}
		\begin{aligned}
			\frac{d \tilde{p}_{2z}}{d \theta_0} = \frac{d \left(-h_2 - p_{2dz}\right)}{d \theta_0} = - \frac{d h_2}{d \theta_0} 
		\end{aligned}
	\end{equation}
	where $p_{2z} = -h_2$ is used.
	The derivative of the inverse function satisfies
	\begin{equation}
		\begin{aligned}
			\frac{d \theta_0}{d \tilde{p}_{2z}} = -
			\frac{\frac{l_1(k - \sin \theta_0)}{g_1^3} + \frac{l_2 (k + \sin \theta_0)}{g_2^3}}{k \cos \theta_0 \left[\frac{4 l_1 l_2}{g_1^3 g_2^3} + 2 l_0 (\frac{l_2}{g_2^3}+\frac{l_1}{g_1^3})\right]}.
		\end{aligned}
	\end{equation}
	
	Finally, differentiating $t_0 \sin \theta_0$ with respect to $\tilde{p}_{2z}$ yields
	\begin{equation}
		\begin{aligned}
			\frac{d (t_0 \sin \theta_0)}{d \tilde{p}_{2z}} &= \frac{d t_0}{d \theta_0} \frac{d \theta_0}{d \tilde{p}_{2z}} \sin \theta_0 + \frac{d \sin \theta_0}{d \tilde{p}_{2z}} t_0\\
			&= -  \frac{d k}{d \theta_0} \frac{d \theta_0}{d \tilde{p}_{2z}}\frac{m_0 g \sin \theta_0}{k^2} + \frac{d \theta_0}{d \tilde{p}_{2z}} \frac{m_0 g \cos \theta_0}{k}\\
			&= - \frac{m_0 g}{k^3 \cos^2 \theta_0} \frac{\frac{l_1}{g_1^3} (k - \sin \theta_0)^2 + \frac{l_2}{g_2^3}(k + \sin \theta_0)^2 + 2 l_0 \sin^2 \theta_0}{\frac{4 l_1 l_2}{g_1^3 g_2^3} + 2 l_0 (\frac{l_1}{g_1^3} + \frac{l_2}{g_2^3})} < 0.
		\end{aligned}
	\end{equation}
	From which the function $\sigma (\cdot)$ is proved to be strictly increasing.
	Based on the constraint function \eqref{constraint2}, the following inequality can be deduced
	\begin{equation}
		\begin{aligned}
			\frac{s}{\cos \theta_0} - 2 l_0 &= \frac{l_1}{\sqrt{(k - \sin \theta_0)^2 + \cos^2 \theta_0}} + \frac{l_2}{\sqrt{(k + \sin \theta_0)^2 + \cos^2 \theta_0}}\\
			&\ge \frac{2 \sqrt{2 l_1 l_2}}{\sqrt{(k-\sin \theta_0)^2 + (k + \sin \theta_0)^2 + 2 \cos^2 \theta_0}}\\
			&= \frac{2 \sqrt{l_1 l_2}}{\sqrt{ k^2 + 1}}			
		\end{aligned}
	\end{equation}
	which implies
	\begin{equation}\label{inequality}
		k^2 \ge  \frac{4 l_1 l_2}{\left(\frac{s}{\cos \theta_0}- 2 l_0\right)^2} - 1.
	\end{equation}
	\begin{assumption}
		Consider the pipe suspended by two quadrotors by cables in the XZ plane of NED frame shown in Fig \ref{fig_22},
		the pitch angle of the pipe $\theta_0$ is assumed to satisfy the following bounded condition
		\begin{equation}
			-\frac{\pi}{2} < - \theta^* \le \theta_0 \le \theta^* < \frac{\pi}{2}
		\end{equation}
		and the cable angles $\alpha$ and $\beta$ are upper bounded by $\kappa$, i.e.,
		\begin{equation}
			\begin{aligned}
				0 < \alpha \le \kappa < \frac{\pi}{2}, \quad \quad 0 < \beta \le \kappa < \frac{\pi}{2},
			\end{aligned}
		\end{equation}
		where $\theta^*$ and $\kappa$ are positive constants.   
	\end{assumption}
	According to the inequality \eqref{inequality}, the horizontal distance $s$ can be adjusted to set the lower bound for $k$,
	i.e.,
	\begin{equation}
		k \ge \underline{k} = \sqrt{\frac{4 l_1 l_2}{\left(\frac{s}{\cos \theta^*} - 2 l_0\right)^2}-1}> 0
	\end{equation}
	where the lower bound $\underline{k}$ corresponds to the upper bound of the internal force $\overline{t}_0$.
	
	For the upper bound $\overline{\sigma}$ of $\frac{d \sigma(x)}{d x}$,
	\begin{equation}
		\begin{aligned}
			\frac{d \sigma (x)}{dx} = -\frac{d (t_0 \sin \theta_0)}{d \tilde{p}_{2z}}
			&=  \frac{m_0 g}{k^3 \cos^3 \theta_0} \frac{l_1 \sin \alpha \cos^2 \alpha + l_2 \sin \beta \cos^2 \beta + 2 l_0 \cos \theta_0 \sin^2 \theta_0}{\frac{4 l_1 l_2}{g_1^3 g_2^3} + 2 l_0 \left(\frac{l_1}{g_1^3} + \frac{l_2}{g_2^3}\right)} \\
			&\le \frac{m_0 g}{k^3 \cos^2 \theta_0} \frac{\left(l_1 \sin \alpha + l_2 \sin \beta + 2 l_0 \cos \theta_0\right) \cdot \max \{\cos^2 \alpha, \cos^2 \beta, \sin^2 \theta_0\}}{2 l_0 \left(\frac{l_1 \sin \alpha}{g_1^2} + \frac{l_2 \sin \beta}{g_2^2}\right)}\\
			&\le \frac{m_0 g}{k^3 \cos^2 \theta_0} \frac{s}{2 l_0 \cdot \frac{l_1 \sin \alpha + l_2 \sin \beta}{k^2 + 2 k +1} } \\
			&\le \frac{m_0 g}{k}\left(1 + \frac{1}{k}\right)^2 \frac{\frac{s}{\cos^2 \theta_0}}{2 l_0 (s - 2 l_0 \cos \theta_0)}\\
			&\le \frac{m_0 g}{\underline{k}} \left(1 + \frac{1}{\underline{k}}\right)^2 \frac{s}{2 l_0 \cos^2 \theta^* (s - 2 l_0)} = \overline{\sigma}
		\end{aligned}
	\end{equation}
	where $g_1^2 \le k^2 + 2 k + 1$ and $g_2^2 \le k^2 + 2 k + 1$ are used.
	\newline
	
	For the lower bound $\underline{\sigma}$ of $\frac{d \sigma(x)}{d x}$,
	\begin{equation}
		\begin{aligned}
			\frac{d \sigma (x)}{dx} = -\frac{d (t_0 \sin \theta_0)}{d \tilde{p}_{2z}}
			&=  \frac{m_0 g}{k^3 \cos^3 \theta_0} \frac{l_1 \sin \alpha \cos^2 \alpha + l_2 \sin \beta \cos^2 \beta + 2 l_0 \cos \theta_0 \sin^2 \theta_0}{\frac{4 l_1 l_2}{g_1^3 g_2^3} + 2 l_0 \left(\frac{l_1}{g_1^3} + \frac{l_2}{g_2^3}\right)} \\
			&\ge m_0 g \frac{l_1 \sin \alpha \cos^2 \alpha + l_2 \sin \beta \cos^2 \beta}{\frac{4 l_1 l_2 k^3 \cos^3 \theta_0}{g_1^3 g_2^3} + 2 l_0 \left(\frac{l_1 k^3 \cos^3 \theta_0}{g_1^3} + \frac{l_2 k^3 \cos^3 \theta_0}{g_2^3}\right)}\\
			&\ge m_0 g \frac{\left(l_1 \sin \alpha + l_2 \sin \beta \right) \cos^2 \max\{\alpha, \beta\}}{\frac{4 l_1 l_2}{\cos^3 \theta_0} + 2 l_0 l_1 + 2 l_0 l_2}\\
			&\ge m_0 g \frac{(s - 2 l_0)  \cos^2 \kappa}{\frac{4 l_1 l_2}{\cos^3 \theta^*} + 2 l_0 l_1 + 2 l_0 l_2} = \underline{\sigma} > 0
		\end{aligned}
	\end{equation}
	where $g_1 \ge k \cos \theta_0$ and $g_2 \ge k \cos \theta_0$ are used.
	\newline
	
	In this manner, $t_0 \sin \theta_0$ can be described as a monotonic function of $\tilde{p}_{2z}$, i.e.,
	\begin{equation}
		\sin \theta_0 = -\sigma \left(\tilde{p}_{2z}\right)
	\end{equation}
	where  $\sigma \left(\cdot\right): \mathbb{R} \to \mathbb{R}$ is a strictly increasing function with $\sigma \left(0\right) = 0$ and satisfies
	\begin{equation}
		\begin{aligned}
			0 <	 \underline{\sigma} < \frac{d \sigma (x)}{d x} < \overline{\sigma}.
		\end{aligned}
	\end{equation}
	Here $\underline{\sigma} =  m_0 g \frac{(s - 2 l_0)  \cos^2 \kappa}{\frac{4 l_1 l_2}{\cos^3 \theta^*} + 2 l_0 l_1 + 2 l_0 l_2}$ and $\overline{\sigma} =  \frac{m_0 g}{\underline{k}} \left(1 + \frac{1}{\underline{k}}\right)^2 \frac{s}{2 l_0 \cos^2 \theta^* (s - 2 l_0)}$.
\end{proof}

\end{document}